\documentclass[twoside]{article}

\usepackage{etoolbox}
\newtoggle{anony}
\toggletrue{anony}

\usepackage{color}
\usepackage{mathtools}
\usepackage{url}
\usepackage{amsthm}
\usepackage{amssymb}
\usepackage{booktabs}
\usepackage{algorithm}
\usepackage{algorithmic}
\usepackage{natbib}
\usepackage{thmtools, thm-restate}

\definecolor{antiquebrass}{rgb}{0.44, 0.47, 0.77}

\usepackage[accepted]{aistats2017}

\newcommand\ACOMM[1]{\STATE {\color{antiquebrass}{// {#1}}}}
\DeclareMathOperator*{\argmin}{arg\,min}

\DeclareMathOperator*{\volume}{volume}
\DeclareMathOperator*{\dist}{dist}
\DeclareMathOperator*{\diss}{diss}
\newcommand\SM{}
\newcommand\trimed{\texttt{trimed}}
\newtheorem{theorem}{Theorem}[section]
\newtheorem{lemma}[theorem]{Lemma}

\newcounter{nbdrafts}
\makeatletter
\newcommand{\checknbdrafts}{
\@latex@warning@no@line{****************************************************}
\ifnum \thenbdrafts > 0
\@latex@warning@no@line{* The document contains \thenbdrafts \space todo note(s)}
\else
\@latex@warning@no@line{* NO DRAFT ANYMORE}
\fi
\@latex@warning@no@line{****************************************************}}

\newcommand{\whp}{w.h.p.\@}

\begin{document}

\twocolumn[

\aistatstitle{A Sub-Quadratic Exact Medoid Algorithm}

\iftoggle{anony}{
\aistatsauthor{ James Newling \And Fran\c cois Fleuret}


\aistatsaddress{ Idiap Research Institute 
\& EPFL\And Idiap Research Institute \& EPFL } 
\toggletrue{anony}
}{
\aistatsauthor{ Anonymous Author 1 \And Anonymous Author 2 \And Anonymous Author 3 }
\aistatsaddress{ Unknown Institution 1 \And Unknown Institution 2 \And Unknown Institution 3 } 
}

]


\begin{abstract}
We present a new algorithm \trimed{} for obtaining the \emph{medoid} of a set, that is the element of the set which minimises the mean distance to all other elements. The algorithm is shown to have, under certain assumptions, expected run time $O( N^{\frac{3}{2}})$ in $\mathbb{R}^d$ where $N$ is the set size, making it the first sub-quadratic exact medoid algorithm for $d>1$. Experiments show that it performs very well on spatial network data, frequently requiring two orders of magnitude fewer distance calculations than state-of-the-art approximate algorithms. As an application, we show how \trimed{} can be used as a component in an accelerated $K$-medoids algorithm, and then how it can be relaxed to obtain further computational gains with only a minor loss in cluster quality. 
\end{abstract}

\section{Introduction}
A popular measure of the centrality of an element of a set is its mean distance to all other elements. In network analysis, this measure is referred to as \emph{closeness centrality}, we will refer to it as \emph{energy}. Given a set $\mathcal{S} = \{x(1), \ldots, x(N)\}$ the energy of element $i \in \{1, \ldots, N\}$ is thus given by, 
\begin{equation*}
E(i) = \frac{1}{N}\sum_{j\in \{1, \ldots, N\}} \dist(x(i), x(j)).
\end{equation*}

An element in $\mathcal{S}$ with minimum energy is referred to as a \emph{1-median} or a \emph{medoid}. Without loss of generality, we will assume that $\mathcal{S}$ contains a unique medoid. The problem of determining the medoid of a set arises in the contexts of clustering, operations research, and network analysis. In clustering, the Voronoi iteration $K$-medoids algorithm~\citep{the_elements, park_2009_kmedoids} requires determining the medoid of each of $K$ clusters at each iteration. In operations research, the facility location problem requires placing one or several facilities so as to minimise the cost of connecting to clients. In network analysis, the medoid may represent an influential person in a social network, or the most central station in a rail network.

\subsection{Medoid algorithms and our contribution}
A simple algorithm for obtaining the medoid of a set of $N$ elements computes the energy of all elements and selects the one with minimum energy, requiring $\Theta(N^2)$ time. In certain settings $\Theta(N)$ algorithms exist, such as in 1-D where the problem is solved by Quickselect~\citep{hoares_algorithm}, and more generally on trees. However, no general purpose $o(N^2)$ algorithm exists. An example illustrating the impossibility of such an algorithm is presented in Supplementary Material B (\SM\ref{app:difficulty}). Related to finding the medoid of a set is finding the \emph{geometric median}, which in vector spaces is defined as the point in the vector space with minimum energy. The relationship between the two problems is discussed in \S\ref{sec:gmedoid}.

Much work has been done to develop approximate algorithms in the context of network analysis. The \texttt{RAND} algorithm of \cite{eppstein_2004_centrality} can be used to estimate the energy of all nodes in a graph. The accuracy of \texttt{RAND} depends on the diameter of the network, which motivated \cite{cohen_at_scale} to use pivoting to make \texttt{RAND} more effective for large diameter networks. The work most closely related to ours is that of \cite{okamoto_2008_centrality}, where \texttt{RAND} is adapted to the task of finding the $k$ lowest energy nodes, $k=1$ corresponding to the medoid problem. The resulting \texttt{TOPRANK} algorithm of \cite{okamoto_2008_centrality} has run time $\tilde{O}(N^{5/3})$ under certain assumptions, and returns the medoid with probability $1 - O(1/N)$, that is \emph{with high probability} (\whp{}). Note that only their run time result requires any assumption, obtaining the medoid \whp{} is guaranteed. \texttt{TOPRANK} is discussed in \S\ref{sec:toprank}.

In this paper we present an algorithm which has expected run time $O(N^{3/2})$ under certain assumptions and always returns the medoid. In other words, we present an exact medoid algorithm with improved complexity over the state-of-the-art approximate algorithm, \texttt{TOPRANK}. We show through experiments that the new algorithm works well for low-dimensional data in $\mathbb{R}^d$ and for spatial network data. Our new medoid algorithm, which we call \trimed{}, uses the triangle inequality to quickly eliminate elements which cannot be the medoid. The $O(N^{3/2})$ run time follows from the surprising result that all but $O(N^{1/2})$ elements can be eliminated in this way. 

The complexity bound on expected run time which we derive contains a term which grows exponentially in dimension $d$, and experiments show that in very high dimensions \trimed{} often ends up computing $O(N^2)$ distances.


\subsection{$K$-medoids algorithms and our contribution}
The $K$-medoids problem is to partition a set into $K$ clusters, so as to minimise the sum over elements of dissimilarites with their nearest medoids. That is, to choose $\mathcal{M} = \{m(1), \ldots, m(K)\} \subset \{1, \ldots, N\}$ to minimise,
\begin{equation*}
\mathcal{L}(\mathcal{M}) = \sum_{i = 1}^{N} \min_{k \in \{1, \ldots, K\}} \diss(x(i), x(m(k))) .
\end{equation*}
We focus on the special case where the dissimilarity is a distance ($\diss = \dist$), which is still more general than $K$-means which only applies to vector spaces. $K$-medoids is used in bioinformatics where elements are genetic sequences or gene expression levels~\citep{clustering_microarray_data} and has been applied to clustering on graphs~\citep{Rattigan_1011}. In machine vision, $K$-medoids is often preferred, as a medoid is more easily interpretable than a mean~\citep{cloudless_rome}. 


The $K$-medoids problem is NP-hard, but there exist approximation algorithms. The Voronoi iteration algorithm, appearing in~\cite{the_elements} and later in~\cite{park_2009_kmedoids}, consists of alternating between updating medoids and assignments, much in the same way as Lloyd's algorithm works for the $K$-means problem. We will refer to it as \texttt{KMEDS}, and to  Lloyd's $K$-means algorithm as \texttt{lloyd}. 

One significant difference between \texttt{KMEDS} and \texttt{lloyd} is that the computation of a medoid is quadratic in the number of elements per cluster whereas the computation of a mean is linear. By incorporating our new medoid algorithm into \texttt{KMEDS}, we break the quadratic dependency of \texttt{KMEDS}, bringing it closer in performance to \texttt{lloyd}. We also show how ideas for accelerating \texttt{lloyd} presented in~\cite{elkan_2003_kmeansicml} can be used in \texttt{KMEDS}. 

It should be noted that algorithms other than \texttt{KMEDS} have been proposed for finding approximate solutions to the $K$-medoids problem, and have been shown to be very effective in~\cite{1609.04723}. These include \texttt{PAM} and \texttt{CLARA} of~\cite{generic_kmedoids}, and \texttt{CLARANS} of \cite{clarans}. In this paper we do not compare cluster qualities of previous algorithms, but focus on accelerating the \texttt{lloyd} equivalent for $K$-medoids as a test setting for our medoid algorithm~\trimed{}. 

\section{Previous works}

\subsection{A related problem: the geometric median}
\label{sec:gmedoid}
A problem closely related to the medoid problem is the geometric median problem. In the vector space $\mathbb{R}^d$ the geometric median, assuming it is unique, is defined as,
\begin{equation}
\label{eqn::geomed}
g(\mathcal{S}) = \argmin_{v \in \mathcal{V}}\left( \sum_{y\in \mathcal{S}} \|v - y\| \right).
\end{equation}
While the medoid of a set is defined in any space with a distance measure, the geometric median is specific to vector spaces, where addition and scalar multiplication are defined. The convexity of the objective function being minimised in~\eqref{eqn::geomed} has enabled the development of fast algorithms. In particular,~\cite{Cohen_GM} present an algorithm which obtains an estimate for the geometric median with relative error $1 + O(\epsilon)$ with complexity $O(nd \log^3(\frac{n}{\epsilon}))$. In $\mathbb{R}^d$, one may hope that such an algorithm can be converted into an exact medoid algorithm, but it is not clear how to do this. 

Thus, while it may be possible that fast geometric median algorithms can provide inspiration in the development of medoid algorithms, they do not work out of the box. Moreover, geometric median algorithms cannot be used for network data as they only work in vector spaces, thus they are useless for the spatial network datasets which we consider in \S\ref{sec:results}.

\subsection{Medoid Algorithms : \texttt{TOPRANK} and \texttt{TOPRANK2}  }
\label{sec:toprank}
In~\cite{eppstein_2004_centrality}, the \texttt{RAND} algorithm for estimating the energy of all elements of a set $\mathcal{S} = \{x(1), \ldots, x(N)\}$ is presented. While \texttt{RAND} is presented in the context of graphs, where the $N$ elements are nodes of an undirected graph and the metric is shortest path length, it can equally well be applied to any set endowed with a distance. The simple idea of \texttt{RAND} is to estimate the energy of each element from a sample of \emph{anchor} nodes  $I$, so that for $j \in \{1, \ldots, N\}$, 
\begin{equation*}
\hat{E}(j) = \frac{1}{|I|}\sum_{i \in I} \dist(x(j),x(i)). 
\end{equation*}
An elegant feature of \texttt{RAND} in the context of sparse graphs is that Dijkstra's algorithm needs only be run from anchor nodes $i \in I$, and not from every node. The key result of~\cite{eppstein_2004_centrality} is the following. Suppose that $\mathcal{S}$ has diameter $\Delta$, that is
\begin{equation*}
\Delta = \max_{(i,j) \in \{1, \ldots, N\}^2}  \dist(x(i),x(j)),
\end{equation*}
and let $\epsilon > 0$ be some error tolerance. If $I$ is of size $\Omega(\log(N) / \epsilon)$, then $\mathbb{P}(|E(j) - \hat{E}(j)| > \epsilon \Delta)$ is $O\left(\frac{1}{N^2}\right)$ for all $j \in \{1, \ldots, N\}$. Using the union bound, this means there is a $O\left(\frac{1}{N}\right)$ probability that at least one energy estimate is off by more than $\epsilon \Delta$, and so we say that \emph{with high probability} (w.h.p.) all errors are less than $\epsilon \Delta$.

\texttt{RAND} forms the basis of the \texttt{TOPRANK} algorithm of~\cite{okamoto_2008_centrality}. Whereas  \texttt{RAND} \whp{} returns an element which has energy within $\epsilon$ of the minimum, \texttt{TOPRANK} is designed to \whp{} return the true medoid. In motivating \texttt{TOPRANK}, \cite{okamoto_2008_centrality} observe that the expected difference between consecutively ranked energies is $O(\Delta/N)$, and so if one wishes to correctly rank all nodes, one needs to distinguish between energies at a scale $\epsilon = \Delta/N$, for which the result of \cite{eppstein_2004_centrality} dictates that  $\Theta (N\log N)$ anchor elements are required with \texttt{RAND}, which is more elements than $\mathcal{S}$ contains. However, to obtain just the highest ranked node should require less information than obtaining a full ranking of nodes, and it is to this task that \texttt{TOPRANK} is adapted. 

The idea behind \texttt{TOPRANK} is to accurately estimate only the energies of promising elements. The algorithm proceeds in two passes, where in the first pass promising elements are earmarked. Specifically, the first pass runs \texttt{RAND} with $N^{2/3}\log^{1/3} (N)$ anchor elements to obtain $\hat{E}(i)$ for $i \in \{1, \ldots, N \}$, and then discards elements whose $\hat{E}(i)$ lies below threshold $\tau$ given by,
\begin{equation}
\label{eq:threshold}
\tau = \argmin_{j \in \{1, \ldots, N\}} \hat{E}(j) + 2  \hat{\Delta} \alpha' \left( \frac{\log n}{n} \right)^{\frac{1}{3}},
\end{equation}
where $\hat{\Delta}$ is an upper bound on $\Delta$ obtained from the anchor nodes, and $\alpha'$ is some constant satisfying $\alpha' > 1$. The second pass computes the true energy of the undiscarded elements, returning the one with lowest true energy. Note that a smaller $\alpha'$ value results in a lower (better) threshold, we discuss this point further in \SM\ref{rand_toprank_2}.

To obtain run time guarantees, \texttt{TOPRANK} requires that the distribution of node energies is non-decreasing near to the minimum, denoted by $E^{*}$. More precisely, letting $f_E$ be the probability distribution of energies, the algorithms require the existence of $\epsilon > 0$ such that,
\begin{equation}
\label{toprankass}
E^{*} \le \tilde{e} < e < E^{*} + \epsilon \implies f_E(\tilde{e}) \le f_E(e).
\end{equation}
If assumption~\ref{toprankass} holds, then the run time is $\tilde{O}(N^{\frac{5}{3}})$. A second algorithm presented in~\cite{okamoto_2008_centrality} is \texttt{TOPRANK2}, where the anchor set $I$ is grown incrementally until some heuristic criterion is met. There is no runtime guarantee for \texttt{TOPRANK2}, although it has the potential to run much faster than \texttt{TOPRANK} under favourable conditions. Pseudocode for \texttt{RAND}, \texttt{TOPRANK} and \texttt{TOPRANK2} is presented in \SM\ref{rand_toprank_2}.


\subsection{$K$-medoids algorithm : \texttt{KMEDS}}
The Voronoi iteration algorithm, which we refer to as \texttt{KMEDS}, is similar to \texttt{lloyd}, the main difference being that cluster medoids are computed instead of cluster means. It has been desribed in the literature at least twice, once in~\cite{the_elements} and then in \cite{park_2009_kmedoids}, where a novel initialisation scheme is developed. Pseudocode is presented in \SM\ref{sec:kmedscode}.

All $N^2$ distances are computed and stored upfront with \texttt{KMEDS}. Then, at each iteration, $KN$ comparisons are made during assignment and $\Omega(N^2/K)$ additions are made during medoid update. The initialisation scheme of \texttt{KMEDS} requires all $N^2$ distances. Each iteration of \texttt{KMEDS} requires retrieving at least $\max\left(KN, N^2/K\right)$ distinct distances, as can be shown by assuming balanced clusters. 

As an alternative to computing all distances upfront, one could store per-cluster distance matrices which get updated on-the fly when assignments change. Using such an approach, the best one could hope for would be $\max\left(KN, N^2/K\right)$ distance calculations and $\Theta(N^2/K)$ memory. If one were to completely forego storing distances in memory and calculate distances only when needed, the number of distance calculations would be at least $r(KN + N^2/K)$, where $r$ is the number of iterations.

The initialisation scheme of~\cite{park_2009_kmedoids} selects $K$ well centered elements as initial medoids. This goes against the general wisdom for $K$-means initialisation, where centroids are initialised to be well separated~\citep{arthur_2007_kmeanspp}. While the new scheme of~\cite{park_2009_kmedoids} performs well on a limited number of small 2-D datasets, we show in \S~\ref{tab:parkinit} that in general uniform initialisation performs as well or better. 

\section{Our new medoid algorithm : \trimed{}}
We present our new algorithm,~\trimed{}, for determining the medoid of set $\mathcal{S} = \{x(1), \ldots, x(N)\}$. Whereas the approach with \texttt{TOPRANK} is to empirically \emph{estimate} $E(i)$ for $i \in \{1, \ldots, N\}$, the approach with \trimed{}, presented as Alg.~\ref{alg::TRITOP}, is to \emph{bound} $E(i)$. When \trimed{} terminates, an index $m^* \in \{1,\ldots,N\}$ has been determined, along with lower bounds $l(i)$ for all $i \in \{1,\ldots,N\}$, such that $E(m^{*}) \le l(i) \le E(i)$, and thus $x(m^{*})$ is the medoid. The bounding approach uses the triangle inequality, as depicted in Figure~\ref{triangles2}.

\begin{algorithm}
\begin{algorithmic}[1]
\STATE $l \gets \underline{0}_N$ {\color{antiquebrass}{\;\;\;// lower bounds on energies, maintained such that $l(i) \le E(i)$ and initialised as $l(i) = 0$. }}
\STATE $m^{cl}, E^{cl} \gets -1, \infty$ {\color{antiquebrass}{\;\;\;// index of best medoid candidate found so far, and its energy.}}
\FOR{$i \in \texttt{shuffle}\left(\{1, \ldots, N\}\right)$} 
\IF{$l(i) < E^{cl}$} 
\FOR{$j \in \{1, \ldots, N \}$}
\STATE $d(j) \gets \dist(x(i),x(j))$
\ENDFOR
\STATE $l(i) \gets \frac{1}{N}\sum_{j = 1}^N d(j)$ {\color{antiquebrass}{\;\;\;// set $l(i)$ to be tight, that is $l(i) = E(i)$.}}
\IF{$l(i) < E^{cl}$}
\STATE $m^{cl}, E^{cl} \gets i, l(i)$
\ENDIF
\FOR{$j \in \{1, \ldots, N\}$}
\STATE $l(j) \gets \max(l(j), |l(i) - d(j)|)$ {\color{antiquebrass}{\;\;\;// using $E(i)$ and $\dist(x(i),x(j))$ to possibly improve bound on $E(j)$.}}
\ENDFOR
\ENDIF
\ENDFOR
\STATE $m^{*}, E^{*} \gets m^{cl}, E^{cl}$
\RETURN $x(m^{*})$
\end{algorithmic}
\caption{The \trimed{} algorithm for computing the medoid of $\{x(1), \ldots, x(N)\}$.}
\label{alg::TRITOP}
\end{algorithm}

The algorithm \trimed{} iterates through the $N$ elements of $\mathcal{S}$. Each time a new element with energy lower than the current lowest energy ($E^{cl}$) is found, the index of the current best medoid ($m^{cl}$) is updated (line 10). Lower bounds on energies are used to quickly eliminate poor medoid candidates (line 4). Specifically, if lower bound $l(i)$ on the energy of element $i$ is greater than or equal to $E^{cl}$, then $i$ is eliminated. If the bound test fails to eliminate element $i$, then it is \emph{computed}, that is, all distances to element $i$ are computed (line 6). The computed distances are used to potentially improve lower bounds for all elements (line 13). Theorem~\ref{thm:reliable} states that \trimed{} finds the medoid. The proof 
relies on showing that lower bounds remain consistent when updated (line 13).

The algorithm is very straightforward to implement, and requires only two additional floating point values per datapoint: for sample $i$, one for $l(i)$ and one for $d(i)$. Computing either all or no distances from a sample makes particularly good sense for network data, where computing all distances to a single node is efficiently performed using Dijkstra's algorithm. 

\begin{figure}
\begin{center}
\includegraphics[width=\columnwidth]{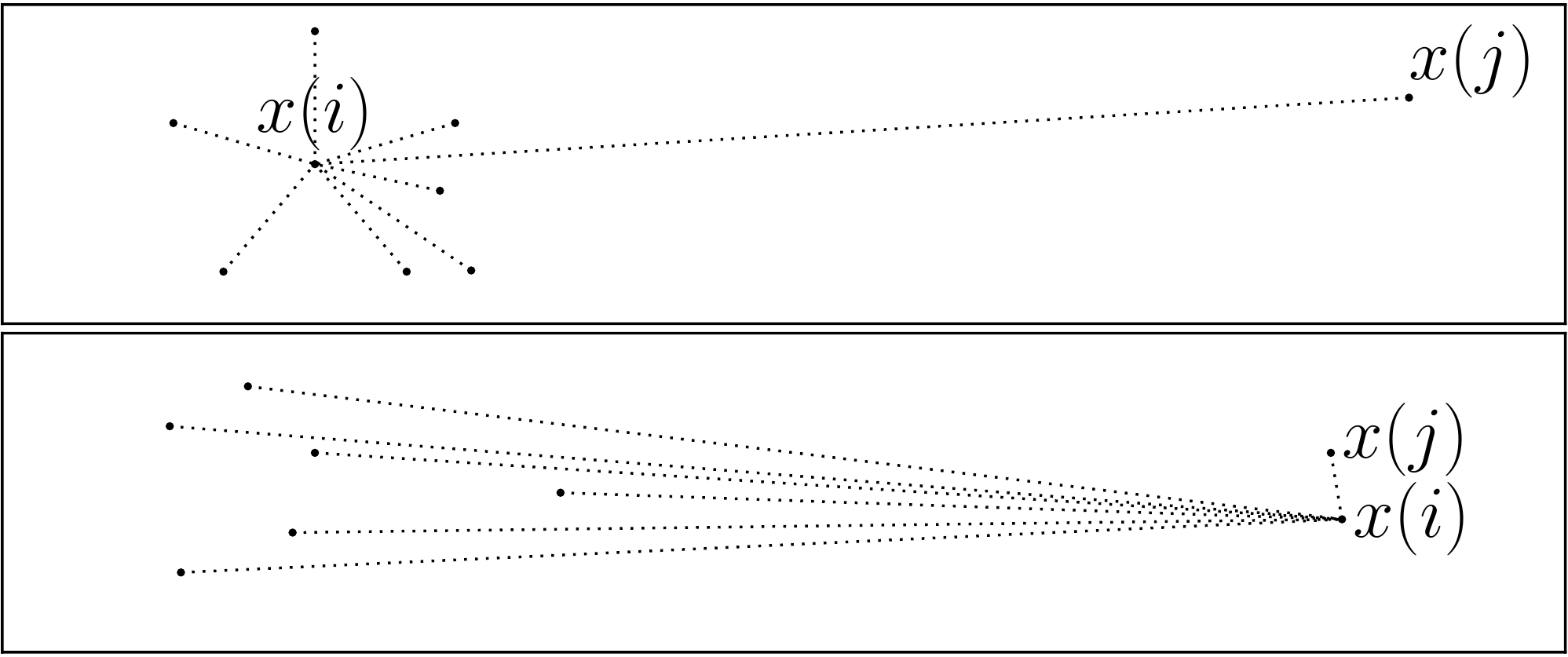}
\caption{Using the inequality $E(j) \ge | E(i) -  \dist(x(i),x(j))\, |$ to eliminate $x(j)$ as a medoid candidate. Computed element $x(i)$ with energy $E(i) \ge E^{cl}$ is used as a pivot to lower bound $E(j)$. The two cases where the inequality is effective are when (case 1, above) $ \dist(x(i),x(j)) - E(i) \ge E^{cl}$ and (case 2, below) $ E(i) - \dist(x(i),x(j))  \ge E^{cl}$, as both lead to $E(j) \ge E^{cl}$ which eliminates $x(j)$ as a medoid candidate.}
\label{triangles2}
\end{center}
\end{figure}

\vspace{1em}

\begin{restatable}{theorem}{reprel}
\label{thm:reliable}
\trimed{} returns the medoid of set $\mathcal{S}$. 
\end{restatable}

\begin{proof}
\label{proof:reliable}
We need to prove that $l(j) \le E(j)$ for all $j \in \{1, \ldots, N\}$ at all iterations of the algorithm. Clearly, as $l(j) = 0$ at initialisation, we have $l(j) \le E(j)$ at initialisation. $E(j)$ does not change, and the only time that $l(j)$ may change is on line 13, where we need to check that $|l(i) - d(j)| \le E(j)$. At line 13, $l(i) = E(i)$  from line 8, and $d(j) = \dist(x(i),x(j))$, so at line 13 we are effectively checking that $|E(i) - \dist(x(i),x(j))| \le E(j)$. But this is a simple consequence of the triangle inequality, as we now show. Using the definition, $ E(j) = \frac{1}{N} \sum_{l = 1}^{N} \dist(x(l),x(j))  $, we have on the one hand,
\begin{align}
 E(j) & \ge  \frac{1}{N} \sum_{l = 1}^{N}  \dist(x(l),x(i)) - \dist(x(i),x(j)) \notag \\
& \ge E(i) - \dist(x(i),x(j)), \label{eqn::dir1}
\end{align}
and on the other hand, 
\begin{align}
E(j) & \ge \frac{1}{N} \sum_{l = 1}^{N} \dist(x(i),x(j)) - \dist(x(l),x(i)) \notag \\
& \ge \dist(x(i),x(j)) - E(i). \label{eqn::dir2} 
\end{align}
Combining~\eqref{eqn::dir1} and~\eqref{eqn::dir2} we obtain the required inequality $|E(i) - \dist(x(i),x(j))| \le E(j)$.
\end{proof}

The bound test (line 4) becomes more effective at later iterations, for two reasons. Firstly, whenever an element is computed, the lower bounds of other samples may increase. Secondly, $E^{cl}$ will decrease whenever a better medoid candidate is found. The main result of this paper, presented as Theorem~\ref{thm:main}, is that in $\mathbb{R}^d$ the expected number of computed elements is $O(N^{\frac{1}{2}})$ under some weak assumptions. We show in \S\ref{sec:results} that the $O(N^{\frac{1}{2}})$ result holds even in settings where the assumptions are not valid or relevent, such as for network data.

The shuffle on line 3 is performed to avoid \whp{} pathological orderings, such as when elements are ordered in descending order of energy which would result in all $N$ elements being computed.

\begin{restatable}[]{theorem}{fta}
\label{thm:main}
Let $\mathcal{S} = \{x(1), \ldots, x(N)\}$ be a set of $N$ elements in $\mathbb{R}^d$, drawn independently from probability distribution function $f_X$. Let the medoid of $\mathcal{S}$ be $x(m^{*})$, and let $E(m^*) = E^{*}$. Suppose that there exist strictly positive constants $\rho, \delta_0$ and $\delta_1$ such that for any set size $N$ with probability $1 - O(1/N)$ 
\begin{equation}
\label{boundyness101}
x \in \mathcal{B}_d(x(m^{*}), \rho) \implies \delta_0 \le f_X(x) \le \delta_1,
\end{equation}
where $\mathcal{B}_d(x, r) = \{x' \in \mathbb{R}^d \;:\; \|x' - x\| \le r\}$. 
Let $\alpha>0$ be a constant (independent of $N$) such that with probability $1 - O(1/N)$ all $i \in \{1, \ldots, N\}$ satisfy,
\begin{align}
\label{strong101}
x(i) \in  \mathcal{B}_d & (x(m^{*}), \rho) \implies  \\
& E(i)  - E^{*} \ge \alpha \|x(i) - x(m^{*}) \|^2. \notag
\end{align}
Then, the expected number of elements computed by \trimed{} is $O\left(\left( V_d[1]\delta_1 + d\left(\frac{4}{\alpha}\right)^d \right) N^{\frac{1}{2}}  \right)$, where $V_d[1] = \pi^{\frac{d}{2}}/(\Gamma{(\frac{d}{2} + 1}))$ is the volume of $\mathcal{B}_d(0, 1)$.
\end{restatable}

\subsection{On the assumptions in Theorem~\ref{thm:main}}
The assumption of constants $\rho, \delta_0$ and $\delta_1$ made in Theorem~\ref{thm:main} is weak, and only pathological distributions might fail it, as we now discuss. For the assumptions to fail requires that $f_X$ vanishes or diverges at the distribution medoid. Any reasonably behaved distribution does not have this behaviour, as illustrated in Figure~\ref{moreexpl}. The constant $\alpha$ is a strong convexity constant. The existence of $\alpha > 0$ is guaranteed by the existence of $\rho, \delta_0$ and $\delta_1$, as the mean of a sum of uniformly spaced cones converges to a quadratic function. This is illustrated in 1-D in Figure~\ref{sumtris} in \SM\ref{sec:bigproof}, but holds true in any dimension. 

Note that the assumptions made are on the distribution $f_X$, and not on the data itself. This must be so in order to prove complexity results in $N$. 

\begin{figure}[ht!]
\begin{center}
\includegraphics[width=\columnwidth]{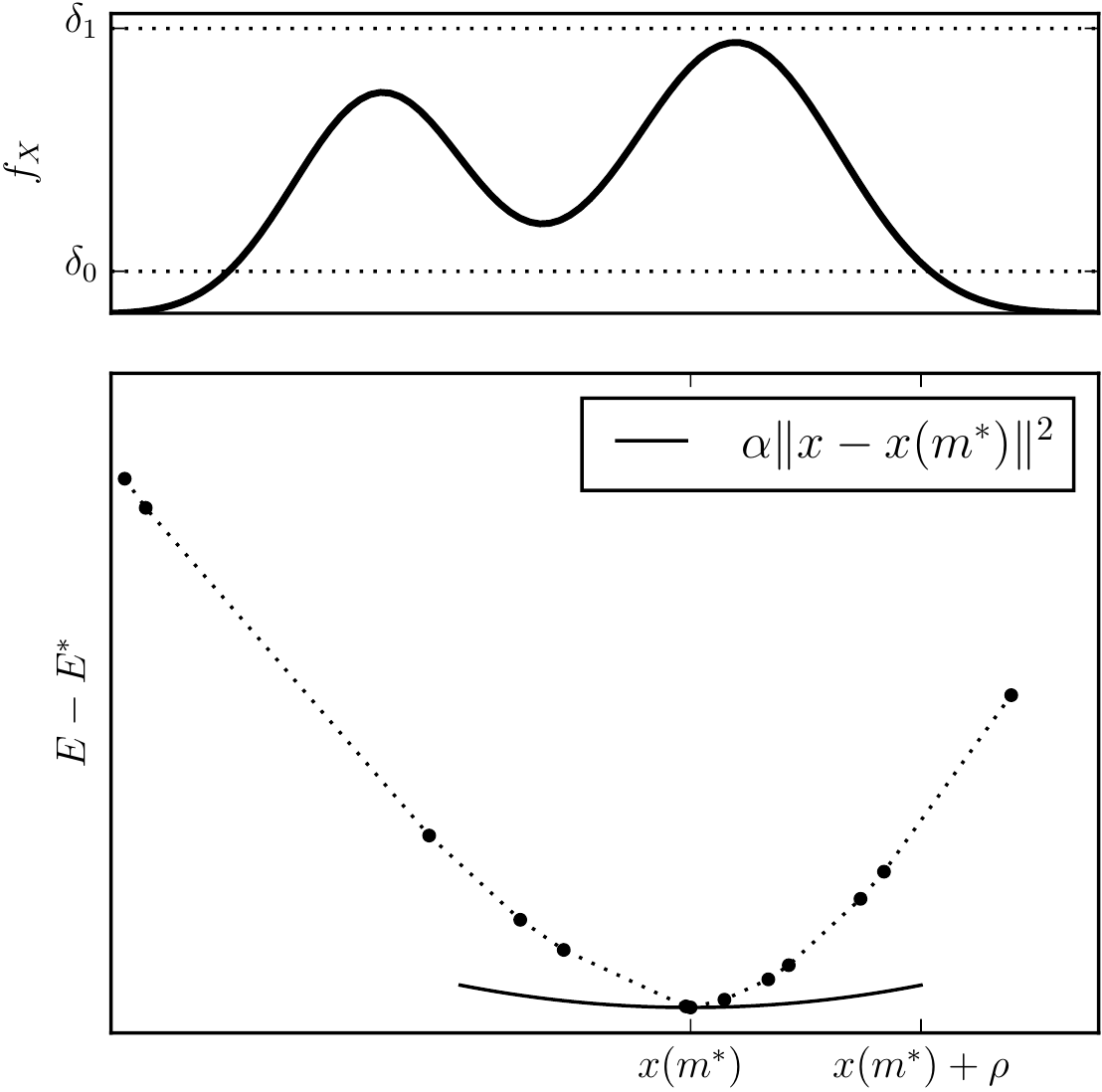}
\caption{Illustration in 1-D of the constants used in Theorem~\ref{thm:main}. Above, $\delta_0$ and $\delta_1$ bound the probability density function in a region containing the distribution medoid. Below, the energy of samples grows quadratically around the medoid $x(m^*)$. The energy $E$ is a sum of cones centered on samples, which is approximately quadratic unless $f_X$ vanishes or explodes, guaranteeing the existence of $\alpha > 0$ required in Theorem~\ref{thm:main}.}
\label{moreexpl}
\end{center}
\end{figure}



\subsection{Sketch of proof of Theorem~\ref{thm:main}}

We now sketch the proof of Theorem~\ref{thm:main}, showing how~\eqref{boundyness101} and \eqref{strong101} are used. A full proof is presented in \SM\ref{sec:bigproof}. Firstly, let the index of the first element after the shuffle on line 3 be $i'$. Then, no elements beyond radius $2E(i')$ of $x(i')$ will subsequently be computed, due to type 1 eliminations (see Figure~\ref{triangles2}). Therefore, all computed elements are contained within $\mathcal{B}_d(x(i'), 2E(i'))$.  

Next, notice that once an element $x(i)$ has been computed in~\trimed{}, no elements in the ball $\mathcal{B}_d(x(i), E(i) - E^{cl})$ will subsequently be computed, due to type 2 eliminations (see Figure~\ref{triangles2}). We refer to such a ball as an \emph{exclusion ball}. By upper bounding the number of exclusion balls contained in $\mathcal{B}_d(x(i'), 2E(i'))$ using a volumetric argument, we can obtain a bound on the number of computed elements,  but obtaining such an upper bound requires that the radii of exclusion ball $E(i) - E^{cl}$ be bounded below by a strictly positive value. However, by using a volumetric argument only beyond a certain positive radius of the medoid (a radius $N^{-1/2d}$), we have $\alpha>0$ in \eqref{strongcon} which provides a lower bound on exclusion ball radii, assuming $E^{cl} \approx E^*$. Using $\delta_0$ we can show that $E^{cl}$ approaches $E^*$ sufficiently fastsufficiently fast to validate the approximation $E^{cl} \approx E^*$.

It then remains to count the number of computed elements within radius $N^{-1/2d}$ of the medoid. One cannot find a strict upper bound here, but using the boundedness of $f_X$ provided by $\delta_1$, we have \whp{} that the number of elements computed within $N^{-1/2d}$ is $O(\delta_1 N^{1/2})$, as the volume of a sphere scales as the $d$'th power of its radius.

\section{Our accelerated $K$-medoids algorithm : \texttt{trikmeds} }
We adapt our new medoid algorithm~\trimed{} and borrow ideas from~\cite{elkan_2003_kmeansicml} to show how~\texttt{KMEDS} can be accelerated. We abandon the initial $N^2$ distance calculations, and only compute distances when necessary. The accelerated version of \texttt{lloyd} of~\cite{elkan_2003_kmeansicml} maintains $KN$ bounds on distances between points and centroids, allowing a large proportion of distance calculations to be eliminated. We use this approach to accelerate assignment in \texttt{trikmeds}, incurring a memory cost $O(KN)$. By adopting the algorithm of~\cite{fast_bounds} or that of~\cite{hamerly_2010_kmeans}, the memory overhead can be reduced to $O(N)$. We accelerate the medoid update step by adapting \trimed{}, reusing lower bounds between iterations, so that \trimed{} is only run from scratch once at the start. Details and pseudocode are presented in \SM\ref{app:trikmeds}. 

One can relax the bound test in \trimed{} so that for $\epsilon_{} > 0$ element $i$ is computed if $l(i)(1 + \epsilon_{}) < E^{cl}$, guaranteeing that an element with energy within a factor $1 + \epsilon_{}$ of $E^{*}$ is found. It is also possible to relax the bound tests in the assignment step of \texttt{trikmeds}, such that the distance to an assigned cluster's medoid is always within a factor $1 + \epsilon_{}$ of the distance to the nearest medoid. We denote by \texttt{trikmeds}-$\epsilon_{}$ the \texttt{trikmeds} algorithm where the update and assignment steps are relaxed as just discussed, with \texttt{trikmeds-0} being exactly \texttt{trikmeds}. The motivation behind such a relaxation is that, at all but the final few iterations, it is probably a waste of computation obtaining medoids and assignments at high resolution, as in subsequent iterations they may change.

\section{Results}
\label{sec:results}
We first compare the performance of the medoid algorithms $\texttt{TOPRANK}, \texttt{TOPRANK2}$ and \trimed{}. We then compare the $K$-medoids algorithms, \texttt{KMEDS} and \texttt{trikmeds}. 

\subsection{Medoid algorithm results}
\label{sec:medresults} 
We compare our new exact medoid algorithm \trimed{} with state-of-the-art approximate algorithms \texttt{TOPRANK} and \texttt{TOPRANK2}. Recall, \cite{okamoto_2008_centrality} prove that the approximate algorithms return \whp{} the true medoid. We confirm that this is the case in all our experiments, where the approximate algorithms return the same element as \trimed{}, which we know to be correct by Theorem~\ref{thm:reliable}. We now focus on comparing computational costs, which are proportional to the number of computed points.
 
Results on artificial datasets are presented in Figure~\ref{rainbow}, where our two main observations relate to scaling in $N$ and dimension $d$. The artificial data are (left) uniformly drawn from $[0,1]^d$ and (right) drawn from $\mathcal{B}_d(0,1)$ with probability of lying within radius $1/2^{1/d}$ of $1/200$, as opposed to $1/2$ as would be the case under uniform density. Details about sampling from this distribution can be found in \SM\ref{sec:cscaling}. Results on a mix of publicly available real and artificial datasets are presented in Table~\ref{philemon} and discussed in \S\ref{res:true}.

\begin{figure*}
\begin{center}
\includegraphics[width=\textwidth]{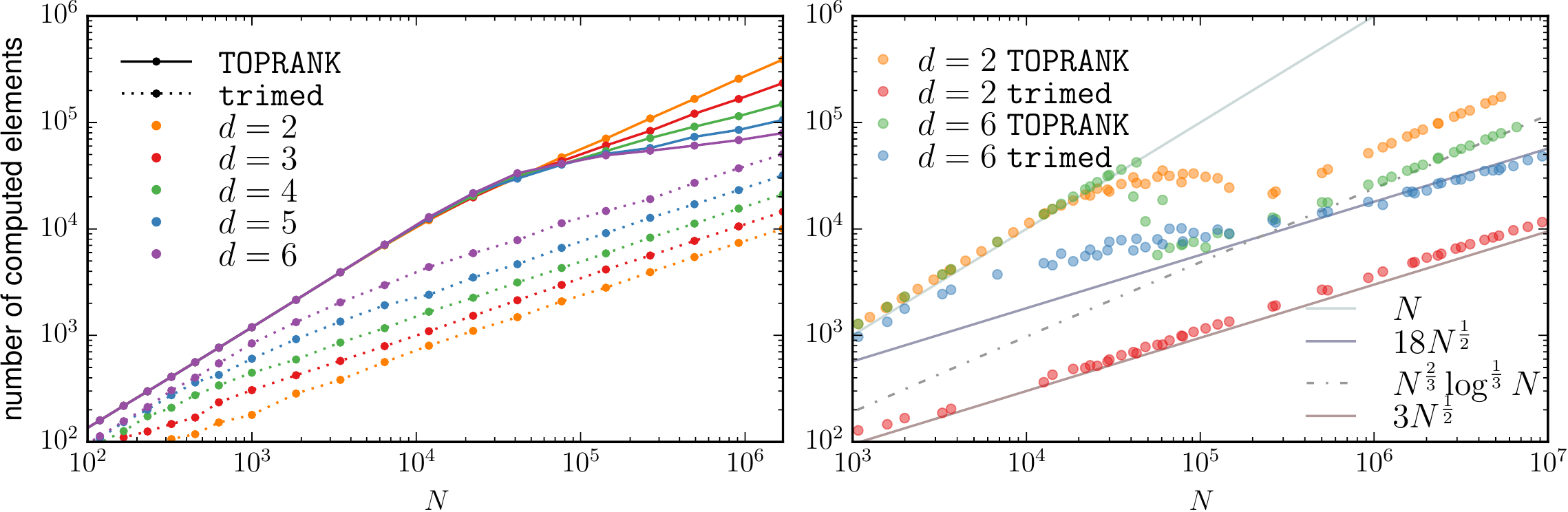}
\caption{Comparison of \texttt{TOPRANK} and our algorithm \trimed{} on simulated data. On the left, points are drawn uniformly from $[0,1]^d$ for $d\in \{2, \ldots, 6\}$, and on the right they are drawn from $\mathcal{B}_d(0,1)$ for $d\in\{2,6\}$, with an increased density near the edge of the ball. Fewer points (elements) are computed by \trimed{} than by \texttt{TOPRANK} in all scenarios. For small $N$, \texttt{TOPRANK} computes $O(N)$ points, before transitioning to $\tilde{O}(N^{2/3})$ computed points for large $N$. \trimed{} computes $O(N^{1/2})$ points. Note that \trimed{} performs better in low-$d$ than in high-$d$, with the reverse trend being true for \texttt{TOPRANK}. These observations are discussed in further detail in the text.}
\label{rainbow}
\end{center}
\vskip -0.1in
\end{figure*}

\subsubsection{Scaling with $N$ and $d$ on artificial datasets}
In Figure~\ref{rainbow} we observe that the number of points computed by \trimed{} is  $O(N^{1/2})$, as predicted by Theorem~\ref{thm:main}. This is illustrated (right) by the close fit of the number of computed points to exact square root curves at sufficiently large $N$ for $d \in \{2,6\}$.

Recall that \texttt{TOPRANK} consists of two passes, a first where $N^{2/3}\log^{1/3} N$ anchor points are computed, and a second where all sub-threshold points are computed. We observe that for small $N$ \texttt{TOPRANK} computes all $N$ points, which corresponds to all points lying below threshold. At sufficiently large $N$ the threshold becomes low enough for all points to be eliminated after the first pass. The effect is particularly dramatic in high dimensions ($d = 6$ on right), where a phase transition is observed between all and no points being computed in the second pass. 

Dimension $d$ appears in Theorem~\ref{thm:main} through a factor $d(4/\alpha)^d$, where $\alpha$ is the strong convexity of the energy at the medoid. In Figure~\ref{rainbow}, we observe that the number of computed points increases with $d$ for fixed $N$, corresponding to a relatively small $\alpha$. The effect of $\alpha$ on the number of computed elements is considered in greater detail in \SM\ref{sec:cscaling}.

In contrast to the above observation that the number of computed points increases as dimension increases for \trimed{}, \texttt{TOPRANK} appears to scale favourably with dimension. This observation can be explained in terms of the distribution of energies, with energies close to $E^{*}$ being less common in higher dimensions, as discussed in \SM\ref{app:topscale}.

\subsubsection{Results on publicly available real and simulated datasets}
\label{res:true} 

We present the datasets used here in detail in \SM\ref{app:datasets}. For all datasets, algorithms \texttt{TOPRANK}, \texttt{TOPRANK2} and \trimed{} were run 10 times with a distinct seed, and the mean number of iterations ($\hat{n}$) over the 10 runs was computed. We observe that our algorithm \trimed{} is the best performing algorithm on all datasets, although in high-dimensions (MNIST-0) and on social network data (Gnutella) no algorithm computes significantly fewer than $N$ elements. The failure in high-dimensions (MNIST-0) of \trimed{} is in agreement with Theorem~\ref{thm:main}, where dimension appears as the exponent of a constant term. The small world network data, Gnutella, can be embedded in a high-dimensional Euclidean space, and thus the failure on this dataset can also be considered as being due to high-dimensions. For low-dimensional real and spatial network data, \trimed{} consistently computes $O(N^{1/2})$ elements. 

\subsubsection{But who needs the exact medoid anyway?}
\label{res:soexact}
A valid criticism that could be raised at this stage would be that for large datasets, finding the exact medoid is probably overkill, as any point with energy reasonably close to $E^{*}$ suffices for most applications. But consider, the \texttt{RAND} algorithm requires computing $\log N / \epsilon^2$ elements to confidently return an element with energy within $\epsilon E^{*}$ of $E^{*}$. For $N = 10^5$ and $\epsilon = 0.05$, this is $4600$, already more than \trimed{} requires to obtain the exact medoid on low-$d$ datasets of comparable size.

\setlength{\tabcolsep}{4.9pt}
\begin{table*}[t]
\begin{center}
\begin{tabular}{|ccc|ccc|}
\hline
\multicolumn{3}{|c|}{  }  & \multicolumn{1}{c|}{ \texttt{TOPRANK} } & \multicolumn{1}{c|}{ \texttt{TOPRANK2} } & \multicolumn{1}{c|}{ \trimed{} }  \\
\hline
\multicolumn{3}{|c|}{  } & \multicolumn{2}{c|}{  } \\[-10pt]
dataset & type & $N$ & $\hat{n}$ & $\hat{n}$ & $\hat{n}$  \\
\hline
Birch 1 & 2-d & $1.0 \times 10^5$ & 57944    & 100180    & \textbf{2180}    \\ 
Birch 2 & 2-d & $1.0 \times 10^5$ & 66062    & 100180    & \textbf{2208}    \\ 
Europe & 2-d & $1.6 \times 10^5$ & 176095    & 169535    & \textbf{2862}    \\ 
U-Sensor Net & u-graph & $3.6 \times 10^5$ & 113838    & 327216    & \textbf{1593}    \\ 
D-Sensor Net & d-graph & $3.6 \times 10^5$ & 99896    & 176967    & \textbf{1372}    \\ 
Pennsylvania road & u-graph & $1.1 \times 10^6$ & 216390    & time-out    & \textbf{2633}    \\ 
Europe rail & u-graph & $4.6 \times 10^4$ & 35913    & 47041    & \textbf{518}    \\ 
Gnutella & d-graph & $6.3 \times 10^3$ & 7043    & 6407    & \textbf{6328}    \\ 
MNIST & 784-d & $6.7 \times 10^3$ & 7472    & 6799    & \textbf{6514}    \\ 
\hline
\end{tabular}
\caption{Comparison of \texttt{TOPRANK}, \texttt{TOPRANK2} and our algorithm \trimed{} on publicly available real and simulated datasets. Column 2 provides the type of the dataset, where `$x$-d' denotes $x$-dimensional vector data, while `d-graph' and `u-graph' denote directed and undirected graphs respectively. Column $\hat{n}$ gives the mean number of elements computed over 10 runs. Our proposed \trimed{} algorithm obtains the true medoid with far fewer computed points in low dimensions and on spatial network data. On the social network dataset (Gnutella) and the very high-$d$ dataset (MNIST), all algorithms fail to provide speed-up, computing approximately $N$ elements.}
\label{philemon}
\end{center}
\vskip 1.0em
\end{table*}

\subsection{$K$-medoids algorithm results}
\label{sec:kmedsres}
With $N$ elements to cluster, \texttt{KMEDS} is $\Theta(N^2)$ in memory, rendering it unusable on even moderately large datasets. To compare the initialisation scheme proposed in~\cite{park_2009_kmedoids} to random initialisation, we have performed experiments on 14 small datasets, with $K \in \{10, \lceil N^{1/2}\rceil, \lceil N/10\rceil \}$. For each of these 42 experimental set-ups, we run the deterministic \texttt{KMEDS} initialisation once, and then uniform random initialisation, 10 times. Comparing the mean final energy of the two initialisation schemes, in only 9 of 42 cases does \texttt{KMEDS} initialisation result in a lower mean final energy. A Table containing all results from these experiments in presented in \SM\ref{park_initialisation}. 

Having demonstrated that random uniform initialisation performs at least as well as the initialisation scheme of \texttt{KMEDS}, and noting that \texttt{trikmeds-0} returns exactly the same clustering as would \texttt{KMEDS} with uniform random initialisation, we turn our attention to the computational performance of~\texttt{trikmeds}. Table~\ref{tab::kmedoidsresults} presents results on 4 datasets, each described in \SM\ref{app:datasets}. The first numerical column is the relative number of distance calculations using \texttt{trikmeds-0} and \texttt{KMEDS}, where large savings in distance calculations, especially in low-dimensions, are observed. Columns $\phi_c$ and $\phi_E$ are the number of distance calculations and energies respectively, using $\epsilon_{} \in \{0.01, 0.1\}$, relative to $\epsilon_{} = 0$. We observe large reductions in the number of distance computations with only minor increases in energy.

\setlength{\tabcolsep}{3.2pt}
\begin{table*}[t]
\begin{center}
\begin{tabular}{|ccc|c|cc|cc|c|cc|cc|}
\hline
\multicolumn{3}{|c|}{  }  & \multicolumn{5}{c|}{ $K = 10$ } & \multicolumn{5}{c|}{ $K = \lceil\sqrt{N}\rceil$ }   \\
\hline
\multicolumn{3}{|c|}{  }  & $\epsilon_{}=0$ & \multicolumn{2}{c|}{ $\epsilon_{} = 0.01$ } & \multicolumn{2}{c|}{ $\epsilon_{} = 0.1$ } & $\epsilon_{}=0$ & \multicolumn{2}{c|}{ $\epsilon_{} = 0.01$ } & \multicolumn{2}{c|}{ $\epsilon_{} = 0.1$ }    \\
\hline
Dataset  & $N$ & $d$ & $N_c/N^2$ & $\phi_c$ & $\phi_E$ & $\phi_c$ & $\phi_E$  & $N_c/N^2$ & $\phi_c$ & $\phi_E$ & $\phi_c$ & $\phi_E$ \\
\hline
Europe & $1.6\times10^5$ & 2 & 0.067 & 0.33 & 1.004 & 0.01 & 1.054  & 0.008 &  0.68 & 1.031 & 0.39  & 1.090 \\
Conflong & $1.6\times10^5$ & 3 & 0.042 & 0.67 & 1.001 & 0.08 & 1.014  & 0.006 &  0.92 & 1.003 & 0.61  & 1.026 \\
Colormo & $6.8\times10^4$ & 9 & 0.163 & 0.92 & 1.000 & 0.35 & 1.015  & 0.011 &  0.98 & 1.000 & 0.82  & 1.005 \\
MNIST50 & $6.0\times10^4$ & 50 & 0.280 & 0.99 & 1.000 & 0.95 & 1.001  & 0.019 &  0.99 & 1.001 & 0.97  & 1.001 \\
\hline
\end{tabular}
\caption{Relative numbers of distance calculations and final energies using \texttt{trikmeds-$\epsilon_{}$} for  $\epsilon_{} \in \{0, 0.01,0.1\}$. The number of distance calculations with \texttt{trikmeds-0} is $N_c$, presented here relative to the number computed using \texttt{KMEDS} ($N^2$) in column $N_c/N^2$. The number of distance calculations with $\epsilon_{} \in \{0.01, 0.1\}$ relative to \texttt{trikmeds-0} are given in columns $\phi_c$, so $\phi_c = 0.33$ means $3\times$ fewer calculations than with $\epsilon = 0$. The final energies with $\epsilon_{} \in \{0.01, 0.1\}$ relative to \texttt{trikmeds-0} are given in columns $\phi_E$. We see that \texttt{trikmeds-0} uses significantly fewer distance calculations than would \texttt{KMEDS}, especially in low-dimensions where a greater than $K\times$ reduction is observed ($N_C/N^2 < 1/K$). For low-$d$, additional relaxation further increases the saving in distance calculations with little cost to final energy.}
\label{tab::kmedoidsresults}
\end{center}
\vskip 1.0em
\end{table*}

\section{Conclusion and future work}

We have presented our new \trimed{} algorithm for computing the medoid of a set, and provided strong theoretical guarantees about its performance in $\mathbb{R}^d$. In low-dimensions, it outperforms the state-of-the-art approximate algorithm on a large selection of datasets. The algorithm is very simple to implement, and can easily be extended to the general ranking problem. In the future, we propose to explore the idea of using more complex triangle inequality bounds involving several points, with as goal to improve on the $O(N^{1/2})$ number of computed points.

We have demonstrated how \trimed{}, when combined with the approach of~\cite{elkan_2003_kmeansicml}, can greatly reduce the number of distance calculations required by the Voronoi iteration $K$-medoids algorithm of~\cite{park_2009_kmedoids}. In the future we would like to replace the strategy of~\cite{elkan_2003_kmeansicml} with that of \cite{hamerly_2010_kmeans}, which will be better adapted to graph clustering as either all or no distances are computed with it, making it more amenable to Dijkstra's algorithm. 

\section*{Acknowledgements}
The authors are grateful to Wei Chen for helpful discussions of the \texttt{TOPRANK} algorithm.
James Newling was funded by the Hasler Foundation under the grant 13018 MASH2.

\bibliographystyle{hapalike}

\bibliography{kmedoids}

\onecolumn
\appendix

\renewcommand{\thesection}{SM-\Alph{section}}




\section{On the difficulty of the medoid problem}
\label{app:difficulty}
We construct an example showing that no general purpose algorithm exists to solve the medoid problem in $o(N^2)$. Consider an almost fully connected graph containing $N = 2m + 1$ nodes, where the graph is exactly $m$ edges short of being fully connected: one node has $2m$ edges and the others have $2m-1$ edges. The graph has $2m^2$ edges.  With the shortest path metric, it is easy to see that the node with $2m$ edges is the medoid, hence the medoid problem is as difficult as finding the node with $2m$ edges. But, supposing that the edges are provided as an unsorted adjacency list, it is clearly an $O(m^2)$ task to determine which node has $2m$ edges as one must look at all edges until a node with $2m$ edges is found. Thus determining the medoid is $O(m^2)$ which is $O(N^2)$.


\section{\texttt{KMEDS} pseudocode}
Alg.~\ref{alg:park} presents the \texttt{KMEDS} algorithm of~\cite{park_2009_kmedoids}, with the novel initialisation of \texttt{KMEDS} on line 1. \texttt{KMEDS} is essentially \texttt{lloyd}, with medoids instead of means. 
\label{sec:kmedscode}
\begin{algorithm}
\begin{algorithmic}[1]
\STATE Set all distances $D(i,j) \gets \|x(i) - x(j)\|$ and sums $S(i) \gets \sum_{j \in \{1, \ldots, N\}} D(i,j)$
\STATE Initialise medoid indices as $K$ indices minimising $f(i) = \sum_{j \in \{1, \ldots, N\}} D(i,j) / S(j) $ 
\WHILE{Some convergence criterion has not been met}
\STATE Assign each element to the cluster whose medoid is nearest to the element
\STATE Update cluster medoids according to assignments made above
\ENDWHILE
\end{algorithmic}
\caption{\texttt{KMEDS} for clustering data $\{x(1), \ldots, x(N)\}$ around $K$ medoids}
\label{alg:park}
\end{algorithm}

\section{\texttt{RAND}, \texttt{TOPRANK} and \texttt{TOPRANK2} pseudocode}
\label{rand_toprank_2} 
We present pseudocode for the \texttt{RAND}, \texttt{TOPRANK} and \texttt{TOPRANK2} algorithms of~\cite{okamoto_2008_centrality}, and discuss the explicit and implicit constants.
\subsection{On the number of anchor elements in \texttt{TOPRANK} : the constant in $\Theta(N^{\frac{2}{3}}\left(\log N \right)^{\frac{1}{3}})$}
\label{app:alphab} 
Note that the number of anchor points used in \texttt{TOPRANK} does not affect the result that the medoid is \whp{} returned. However, \cite{okamoto_2008_centrality} show that by choosing the size of the anchor set to be $q\left(\log N \right)^{\frac{1}{3}}$ for any $q$, the run time is guaranteed to be $\tilde{O}(N^{5/3})$.  They do not suggest a specific $q$, the optimal $q$ being dataset dependant. We choose $q = 1$.

Consider Figure~\ref{rainbow} in Section~\ref{sec:medresults} for example, where $q = 1$. Had $q$ be chosen to be less than $1$, the line $\texttt{ncomputed} = N^{2/3}\log^{1/3}N$ to which \texttt{TOPRANK} runs parallel for large $N$ would be shifted up or down by $\log q$, however the $N$ at which the transition from $\texttt{ncomputed} = N^{2/3}\log^{1/3}N$ to $\texttt{ncomputed} = N^{2/3}\log^{1/3}N$ takes place would also change. 

\subsection{On the parameter $\alpha'$ in \texttt{TOPRANK} and \texttt{TOPRANK2}}
The threshold $\tau$ in~\eqref{eq:threshold} is proportional to the parameter $\alpha'$. In~\cite{okamoto_2008_centrality}, it is stated that $\alpha'$ should be some value greater than 1. Note that the smaller $\alpha'$ is, the lower the threshold is, and hence fewer the number of computed points is, thus $\alpha' = 1.00001$ would be a fair choice.  We use $\alpha' = 1$ in our experiments, and observe that the correct medoid is returned in all experiments. 

Personal correspondence with the authors of~\cite{okamoto_2008_centrality} has brought into doubt the proof of the result that the medoid is \whp{} returned for any $\alpha'$ where $\alpha' > 1$. In our most recent correspondence, the authors suggest that the \whp{} result can be proven with the more conservative bound of $\alpha' > \sqrt{1.5}$. Moreover, we show in \SM\ref{sec:lerror} that  $\alpha' > 1$ is good enough to return the medoid with probability $N^{-(\alpha' - 1)}$, a probability which still tends to $0$ as $N$ grows large, but not a \whp{} result. Please refer to \SM\ref{sec:lerror} for further details on our correspondence with the authors.

\subsection{On the parameters specific to \texttt{TOPRANK2}}
In addition to $\alpha'$, \texttt{TOPRANK2} requires two parameters to be set. The first is $l_0$, the starting anchor set size, and the second is $q$, the amount by which $l$ should be incremented at each iteration. \cite{okamoto_2008_centrality} suggest taking $l_0$ to be the number of top ranked nodes required, which in our case would be $l_0 = k = 1$. However, in our experience this is too small as all nodes lie well within the threshold and thus when $l$ increases there is no change to number below threshold, which makes the algorithm break out of the search for the optimal $l$ too early. Indeed, $l_0$ needs to be chosen so that at least some points have energies greater than the threshold, which in our experiments is already quite large. We choose $l_0 = \sqrt{N}$, as any value larger than $N^{2/3}$ would make \texttt{TOPRANK2} redundant to \texttt{TOPRANK}. The parameter $q$ we take to be $\log{N}$ as suggested by~\cite{okamoto_2008_centrality}.

\begin{algorithm}
\begin{algorithmic}
\STATE $I \gets $ random uniform sample from $\{1, \ldots, N\}$
\ACOMM{Compute all distances from anchor elements ($I$), using Dijkstra's algorithm on graphs}
\FOR{$i \in I$} 
\FOR{$j \in \{1, \ldots, N \}$}
\STATE $d(i,j) \gets \|x(i) - x(j)\|$, 
\ENDFOR
\ENDFOR
\ACOMM{Estimate energies as mean distances to anchor elements}
\FOR{$j \in \{1, \ldots, N \}$}
\STATE $\hat{E}(j) \gets \frac{1}{|I|}\sum_{i\in I}d(i,j)$
\ENDFOR
\RETURN $\hat{E}$
\end{algorithmic}
\caption{\texttt{RAND} for estimating energies of elements of set $S$~\citep{eppstein_2004_centrality}.}
\label{alg::RAND}
\end{algorithm}

\begin{algorithm}
\begin{algorithmic}
\STATE $l \gets N^{\frac{2}{3}} \left(\log N \right)^{\frac{1}{3}}$ {\color{antiquebrass}{\;\;\;// \cite{okamoto_2008_centrality} state that $l$ should be $\Theta(\left(\log N \right)^{\frac{1}{3}})$, the choice of 1 as the constant is arbitrary (see comments in the text of Section~\ref{app:alphab}). }} 
\STATE Run \texttt{RAND} with uniform random $I$ of size $l$ to get $\hat{E}(i)$ for $i \in \{1, \ldots, N\}.$
\STATE Sort $\hat{E}$ so that $\hat{E}[1] \le \hat{E}[2] \le \ldots \le \hat{E}[N]$
\STATE  $\hat{\Delta} \gets 2 \min_{i \in I} \max_{j \in \{1, \ldots, N\}} \|x(i) - x(j)\|$ {\color{antiquebrass}{\;\;\;// where $\|x(i) - x(j)\|$ computed in \texttt{RAND} }}
\STATE $Q \gets \left\{ i \in \{1, \ldots, N \} \;|\; \hat{E}(i) \le \hat{E}[k] + 2 \alpha' \Delta\sqrt{\frac{\log(n)}{l}} \right\}$.
\STATE Compute exact energies of all elements in $Q$ and return the element with the lowest energy.
\end{algorithmic}
\caption{\texttt{TOPRANK} for obtaining top $k$ ranked elements of $S$~\citep{okamoto_2008_centrality}.}
\label{alg::TOPRANK}
\end{algorithm}

\begin{algorithm}
\begin{algorithmic}
\ACOMM{\color{antiquebrass}{In~\cite{okamoto_2008_centrality}, it is suggested that $l_0$ be taken as $k$, which in the case of the medoid problem is $1$. We have experimented with several choices for $l_0$, as discussed in the text. }}
\STATE $l \gets l_0$ 
\STATE Run \texttt{RAND} with uniform random $I$ of size $l$ to get $\hat{E}(i)$ for $i \in \{1, \ldots, N\}.$
\STATE  $\hat{\Delta} \gets 2 \min_{i \in I} \max_{j \in \{1, \ldots, N\}} \|x(i) - x(j)\|$ {\color{antiquebrass}{\;\;\;// where $\|x(i) - x(j)\|$ computed in \texttt{RAND} }}
\STATE Sort $\hat{E}$ so that $\hat{E}[1] \le \hat{E}[2] \le \ldots \le \hat{E}[N]$
\STATE $Q \gets \left\{ i \in \{1, \ldots, N \} \;|\; \hat{E}(i) \le \hat{E}[k] + 2 \alpha' \Delta\sqrt{\frac{\log(n)}{l}} \right\}$.
\STATE $g \gets 1$  
\WHILE{$g$ is $1$}
\STATE $p \gets |Q|$
\ACOMM{\color{antiquebrass}{The recommendation for $q$ in~\cite{okamoto_2008_centrality} is $\log(n)$, we follow the suggestion}}
\STATE Increment $I$ with $q$ new anchor points 
\STATE Update $\hat{E}$ for all data according to new anchor points
\STATE $l \gets |I|$
\STATE $\hat{\Delta} \gets 2 \min_{i \in I} \max_{j \in \{1, \ldots, N\}} \|x(i) - x(j)\|$
\STATE Sort $\hat{E}$ so that $\hat{E}[1] \le \hat{E}[2] \le \ldots \le \hat{E}[N]$
\STATE $Q \gets \left\{ i \in \{1, \ldots, N \} \;|\; \hat{E}(i) \le \hat{E}[k] + 2 \alpha' \Delta\sqrt{\frac{\log(n)}{l}} \right\}$
\STATE $p' \gets |Q|$
\IF{$p - p' < \log{(n)} $}
\STATE $g \gets 0$
\ENDIF
\ENDWHILE
\STATE Compute exact energies of all elements in $Q$ and return the element with the lowest energy

\end{algorithmic}
\caption{\texttt{TOPRANK2} for obtaining top $k$ ranked elements of $S$~\citep{okamoto_2008_centrality}.}
\label{alg::TOPRANK2}
\end{algorithm}

\section{On the proof that \texttt{TOPRANK} returns the medoid with high probability}
\label{sec:lerror}
Through correspondence with the authors of~\cite{okamoto_2008_centrality}, we have located a small problem in the proof that the medoid is returned \whp{} for $\alpha' > 1$, the problem lying in the second inequality of Lemma 1. To arrive at this inequality, the authors have used the fact that for all $i$,
\begin{equation}
\label{truestatement}
\mathbb{P}(E(i) \ge \hat{E}(i) + f(l) \cdot \Delta) \ge 1 - \frac{1}{2N^2},
\end{equation}
which is a simple consequence of the Hoeffding inequality as shown in~\cite{eppstein_2004_centrality}. Essentially~\eqref{truestatement} says that, for a fixed node $i$, from which the mean distance to other nodes is $E(i)$, if one uniformly samples $l$ distances to $i$ and computes the mean $\hat{E}(i)$, the probability that $\hat{E}(i)$ is less than $E(i) + f(l)$ is greater than $1 - \frac{1}{2N^2}$.

The inequality~\eqref{truestatement} is true for a fixed node $i$. However, it no longer holds if $i$ is selected to be the node with the lowest $\hat{E}(i)$. To illustrate this, suppose that $E(i) = 1$ for all $i$, and compute $\hat{E}(i)$ for all $i$. Let $\hat{E}^* = \argmin_i \hat{E}(i)$. Now, we have a strong prior on $\hat{E}^*$ being significantly less than $1$, and~\eqref{truestatement} no longer holds as a statement made about $\hat{E}^*$.

In personal correspondence, the authors show that the problem can be fixed by the use of an additional layer of union bounding, with a correction to be published (if not already done so at time of writing). However, the additional layer of union bound requires a more conservative constraint on $\alpha'$, which is $\alpha' > 2$, although the authors propose that the \whp{} result can be proven with $\alpha' > \sqrt{1.5}$ for $N$ sufficiently large. We now present a small proof proving the \whp{} result for $\alpha' > \sqrt{2}$ for $N$ sufficiently large, with at the same time $\alpha' > 1$ guaranteeing that the medoid is returned with probability $O(N^{\alpha' -1})$.

\subsection{That the medoid is returned \emph{with high probability} holds for $\alpha' > \sqrt{2}$ and that \emph{with vanishing probability} it is returned for $\alpha' > 1$}
Recall that we have $N$ nodes with \emph{energies} $E(1),\ldots, E(n)$. We wish to find the $k$ lowest energy nodes (the original setting of~\cite{okamoto_2008_centrality}). From Hoeffding's inequality we have,
\begin{equation}
\label{eq1}
\mathbb{P}(|E(i) - \hat{E}(i)| \ge \epsilon \Delta ) \le 2\exp{ \left(-l \epsilon^2 \right)}.\\
\end{equation}
Set the probability on the right hand side of~\ref{eq1} to be $2/N^{1 + \beta}$, that is,
\begin{equation*}
2\exp{ \left(-l \epsilon^2 \right)} = 2/N^{1 + \beta},
\end{equation*}
which corresponds to 
\begin{equation*}
\epsilon = \sqrt{\left( \frac{1 + \beta}{l}\right) \log{(N)}}\;\; \coloneqq \;\; \tilde{f}(l).
\end{equation*}
Clearly $\sqrt{1 + \beta}$ corresponds to $\alpha'$. With this notation we have,
\begin{equation}
\label{eq2}
\mathbb{P}(|E(i) - \hat{E}(i)| \ge \tilde{f}(l) \Delta ) \le \frac{2}{N^{1+\beta}}.\\
\end{equation}
Applying the union bound to~\eqref{eq2} we have,
\begin{equation}
\label{eqn}
\mathbb{P} \left( \lnot \left( \displaystyle \wedge_{i \in \{1, \ldots, N\}}  |E(i) - \hat{E}(i)| \le \tilde{f}(l) \Delta \right)  \right)  \le \frac{2}{N^{\beta}}.\\
\end{equation}

Recall that we wish to obtain the $k$ nodes with lowest energy. Denote by $r(j)$ the index of the node with the $j$'th lowest energy, so that
\begin{equation*}
E(r(1)) \le \ldots \le E(r(j)) \le \ldots \le E(r(N)).
\end{equation*}
Denote by $\hat{r}(j)$ the index of the node with the $j$'th lowest estimated energy, so that 
\begin{equation*}
\hat{E}(\hat{r}(1)) \le \ldots \le \hat{E}(\hat{r}(j)) \le \ldots \le \hat{E}(\hat{r}(N)).
\end{equation*}
Now assume that for all $i$, it is true that $ |E(i) - \hat{E}(i)| \le \tilde{f}(l)$. Then consider, for $j \le k$, 
\begin{align}
\hat{E}({\hat{r}(k)}) - \hat{E}(r(j)) & = 
\underbrace{\Big(\hat{E}({\hat{r}(k)}) - E(r(k))\Big)}_{\ge - \tilde{f}(l) \Delta} + \underbrace{\Big( E(r(k)) -  E(r(j))\Big)}_{\ge 0} + \underbrace{\Big(E(r(j)) - \hat{E}(r(j))\Big)}_{\ge - \tilde{f}(l)\Delta},\label{eq3} \\
& \ge  -2\tilde{f}(l)\Delta. \notag
\end{align}
The first bound in~\eqref{eq3} is obtained by considering the most extreme case possible under the assumption, which is $\hat{E}(i) = a(E) - \tilde{f}(l)$ for all $i$. The second bound follows from $j \le k$, and the third bound follows directly from the assumption. We thus have that, under the assumption, 
\begin{equation*}
\hat{E}(r(j)) \le \hat{E}({\hat{r}(k)}) + 2\tilde{f}(l)\Delta,
\end{equation*}
which says that all nodes of rank less than or equal to $k$ have approximate energy less than $\hat{E}({\hat{r}(k)}) + 2\tilde{f}(l)\Delta$.  As the assumption holds with probability greater than $1 - 2/N^{\beta}$ by~\eqref{eqn}, we are done. Take $\beta = 1$ if you want the statement \emph{with high probability}, that is 
\begin{equation*}
\epsilon =  \sqrt{\frac{2\log(n)}{l}},
\end{equation*}
but for any $\beta > 0$, which corresponds to $\alpha' > 1$, the probability of failing to return the $k$ lowest energy nodes tends to $0$ as $N$ grows.

\section{On the initialisation of~\cite{park_2009_kmedoids}}
\label{park_initialisation}

\begin{table}
\begin{center}
\begin{tabular}{|ccc|cc|cc|cc|}
\hline
\multicolumn{3}{|c|}{  }  & \multicolumn{2}{c|}{ $K = 10$ } & \multicolumn{2}{c|}{ $K = \left\lceil\sqrt{N}\right\rceil$ }  & \multicolumn{2}{c|}{ $K = \left\lceil\frac{N}{10}\right\rceil$ }   \\
\hline
Dataset & $N$ & $d$ & $\mu_{\texttt{u}}/\mu_{\texttt{park}}$ & $\sigma_{\texttt{u}}/\mu_{\texttt{park}}$ &  $\mu_{\texttt{u}}/\mu_{\texttt{park}}$ & $\sigma_{\texttt{u}}/\mu_{\texttt{park}}$ & $\mu_{\texttt{u}}/\mu_{\texttt{park}}$ & $\sigma_{\texttt{u}}/\mu_{\texttt{park}}$ \\
\hline
gassensor & 256 & 128 & 1.09 & 0.08 & $\mathbf{0.90}$ & 0.03 & $\mathbf{0.83}$ & 0.01 \\ 
house16H & 1927 & 17 & 1.01 & 0.02 & $\mathbf{0.97}$ & 0.01 & $\mathbf{0.93}$ & 0.01 \\ 
S1 & 5000 & 2 & 1.05 & 0.05 & $\mathbf{0.75}$ & 0.01 & $\mathbf{0.32}$ & 0.01 \\ 
S2 & 5000 & 2 & 1.04 & 0.07 & $\mathbf{0.68}$ & 0.01 & $\mathbf{0.34}$ & 0.00 \\ 
S3 & 5000 & 2 & 1.03 & 0.05 & $\mathbf{0.76}$ & 0.01 & $\mathbf{0.35}$ & 0.00 \\ 
S4 & 5000 & 2 & 1.02 & 0.03 & $\mathbf{0.75}$ & 0.01 & $\mathbf{0.41}$ & 0.01 \\ 
A1 & 3000 & 2 & $\mathbf{0.82}$ & 0.03 & $\mathbf{0.43}$ & 0.01 & $\mathbf{0.19}$ & 0.00 \\ 
A2 & 5250 & 2 & $\mathbf{0.98}$ & 0.03 & $\mathbf{0.47}$ & 0.01 & $\mathbf{0.25}$ & 0.00 \\ 
A3 & 7500 & 2 & $\mathbf{0.96}$ & 0.02 & $\mathbf{0.42}$ & 0.02 & $\mathbf{0.22}$ & 0.00 \\ 
thyroid & 215 & 5 & $\mathbf{0.95}$ & 0.08 & $\mathbf{0.97}$ & 0.04 & $\mathbf{0.93}$ & 0.04 \\ 
yeast & 1484 & 8 & 1.00 & 0.02 & $\mathbf{0.96}$ & 0.02 & $\mathbf{0.91}$ & 0.02 \\ 
wine & 178 & 14 & 1.01 & 0.02 & 1.02 & 0.01 & $\mathbf{0.98}$ & 0.02 \\ 
breast & 699 & 9 & $\mathbf{0.79}$ & 0.03 & $\mathbf{0.77}$ & 0.02 & $\mathbf{0.68}$ & 0.02 \\ 
spiral & 312 & 3 & 1.03 & 0.03 & $\mathbf{0.99}$ & 0.02 & $\mathbf{0.82}$ & 0.03 \\ 
\hline
\end{tabular}
\end{center}
\caption{Comparing the initialisation scheme proposed in~\cite{park_2009_kmedoids} with random uniform initialisation for the \texttt{KMEDS} algorithm. The final energy using the deterministic scheme proposed in~\cite{park_2009_kmedoids} is $\mu_{\texttt{park}}$. The mean over 10 random uniform initialisations is $\mu_{u}$, and the corresponding standard deviation is $\sigma_u$. For small $K$ ($K = 10)$, the performances using the two schemes are comparable, while for larger $K$, it is clear that uniform initialisation performs much better on the majority of datasets. }
\label{tab:parkinit}
\end{table}

In Table~\ref{tab:parkinit} we present the full results of the 48 experiments comparing the initialisation proposed in~\cite{park_2009_kmedoids} with simple uniform initialisation. The 14 datasets are all available from \url{https://cs.joensuu.fi/sipu/datasets/}.

\section{Scaling with $\alpha$, $N$, and dimension $d$}
\label{sec:cscaling}
We perform more experiments to provide further validation of Theorem~\ref{thm:main}. In particular, we check how the number of computed elements scales with $N$, $d$ and $\alpha$. We generate data from a unit ball in various dimensions, according to two density functions with different strong convexity constants $\alpha$. The first density function is uniform, so that the density everywhere in the ball is uniform. To sample from this distribution, we generate two random variables, $X_1 \sim \mathcal{N}_d(0, 1)$ and $X_2 \sim U(0,1)$ and use 
\begin{equation}
\label{uniformball}
X_3 = X_1/\|X_1\| \cdot X_2^{\frac{1}{d}},
\end{equation}
as a sample from the unit ball $\mathcal{B}_d(0,1)$ with uniform distribution.  The second distribution we consider has a higher density beyond radius $(1/2)^{1/d}$. Specifically, within this radius the density is $19\times$ lower than beyond this radius. To sample from this distribution, we sample $X_3$ according to~\eqref{uniformball}, and then points lying within radius $(1/2)^{1/d}$ are with probability $1/10$ re-sampled uniformly beyond this radius. 

\begin{figure}
\begin{center}
\includegraphics[width=\textwidth]{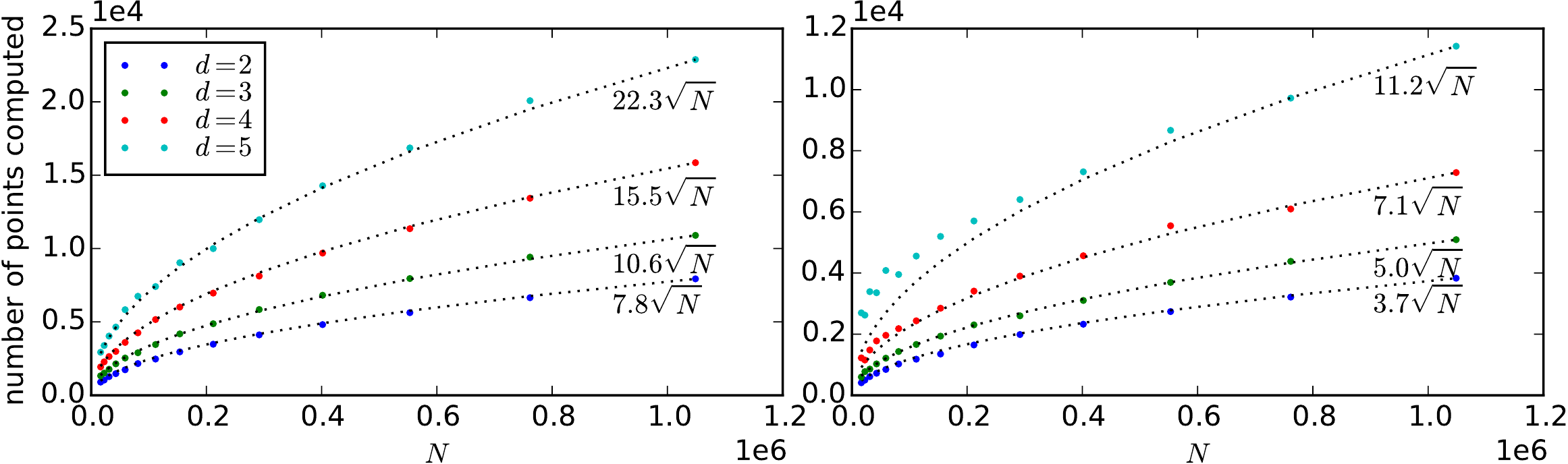}
\vskip 0.1in
\end{center}
\caption{Number of points computed on simulated data. Points are drawn from $\mathcal{B}_d(0, 1)$, for $d \in \{2,3,4,5\}$. On the left, points are drawn uniformly, while on the right, the density in $\mathcal{B}_d(0, (1/2)^{1/d})$ is $19\times$ lower that in $\mathcal{A}_d(0, (1/2)^{1/2}, 1)$, where recall that $\mathcal{A}_d(x, r_1, r_2)$ denotes an annulus centred at $x$ of inner radius $r_1$ and outer radius $r_2$. We observe a near perfect fit of the number of computed points to $\xi\sqrt{N}$ where the constant $\xi$ depends on the dimension and the distribution (left and right). The number of computed points increases with dimension. The strong convexity constant of the distribution on the right is larger, corresponding to fewer distance calculations as predicted by Theorem~\ref{thm:main}.}
\label{scalx1}
\end{figure}

The second distribution has a larger strong convexity constant $\alpha$. To see this, note that the strong convexity constant at the center of the ball depends only on the density of the ball on its surface, that is at radius 1, as can be shown using an argument based on cancelling energies of internal points. As the density at the surface under distribution 2 is approximately twice that of under distribution 1, the change in energy caused by a small shift in the medoid is twice as large under distribution 2. Thus, according to Theorem~\ref{thm:main}, we expect the number of computed points to be larger under distribution 1 than under distribution 2. This is what we observe, as shown in Figure~\ref{scalx1}, where distribution 1 is on the left and distribution 2 is on the right.

In Figure~\ref{scalx1} we observe a near perfect $N^{1/2}$ scaling of number of computed points. Dashed curves are exact $N^{1/2}$ relationships, while the coloured points are the observed number of computed points.


\section[Proof of The Main Theorem]{Proof of Theorem~\ref{thm:main} (See page~\pageref{thm:main})}
\label{sec:bigproof}
\fta*
\begin{proof}
We show that the assumptions made in Th.~\ref{thm:main} validate the assumptions required in Thm~\ref{thm:mainsub}. Firstly, if $e(i) > \rho$ then $e(i) \ge \alpha \rho^2 e(i)  > \rho $, which follows from the convexity of the loss function and. Secondly, the existance of $\beta$ follows from continuity of the gradient of the distance, combined with the existence of $\delta_1$ (non-exploding). 
\end{proof}

\begin{restatable}[Main Theorem Expanded]{theorem}{ftasub}
\label{thm:mainsub}
Let $\mathcal{S} = \{x(1), \ldots, x(N)\} \subset \mathbb{R}^d$ have medoid $x(m^{*})$ with minimum energy $E(m^*) = E^{*}$, where elements in $\mathcal{S}$ are drawn independently from probability distribution function $f_X$. Let $e(i) = \|x(i) - x(m^{*}) \|$. Suppose that for $f_X$ there exist strictly positive constants $\alpha, \beta, \rho, \delta_0$ and $\delta_1$ satisfying,   
\begin{equation}
\label{boundyness}
x \in \mathcal{B}_d(x(m^{*}), \rho) \implies \delta_0 \le f_X(x) \le \delta_1,
\end{equation}
where $\mathcal{B}_d(x, r) = \{x' \in \mathbb{R}^d \;:\; \|x' - x\| \le r\}$, and that for any set size $N$, \whp{}  all $i \in \{1, \ldots, N\}$ satisfy,
\begin{align}
E(i)  - E^{*} \ge & \begin{cases}
\label{strongcon}
\alpha e(i)^2 & \mbox{ if \;\;  } e(i) \le \rho,\\
\alpha \rho^2 & \mbox{ if \;\; }  e(i)  > \rho, 
\end{cases}
\end{align}
and,
\begin{align}
\label{final_extension}
E(i) - E^{*}  \le  \beta e(i)^2\;\; \mbox{ \;  if \;\; } e(i) \le \rho.
\end{align}
Then the expected number of elements computed, which is to say not eliminated on line 4 of \trimed{}, is $O\left(\left( V_d[1]\delta_1 + d\left(\frac{4}{\alpha}\right)^d \right) N^{\frac{1}{2}}  \right)$, where $V_d[1] = \pi^{\frac{d}{2}}/(\Gamma{(\frac{d}{2} + 1}))$ is the volume of $\mathcal{B}_d(0, 1)$.
\end{restatable}

\begin{figure}[h!]
\begin{center}
\includegraphics[width=0.7\textwidth]{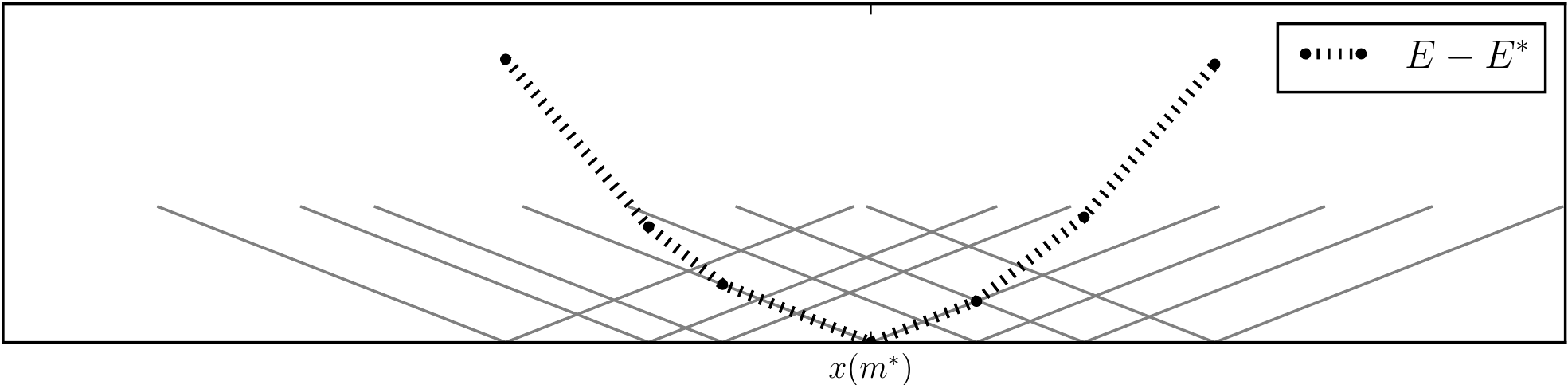}
\caption{A sum of uniformly distributed cones is approximately quadratic.}
\label{sumtris}
\end{center}
\end{figure}

\begin{figure}[h!]
\begin{center}
\includegraphics[width=0.7\textwidth]{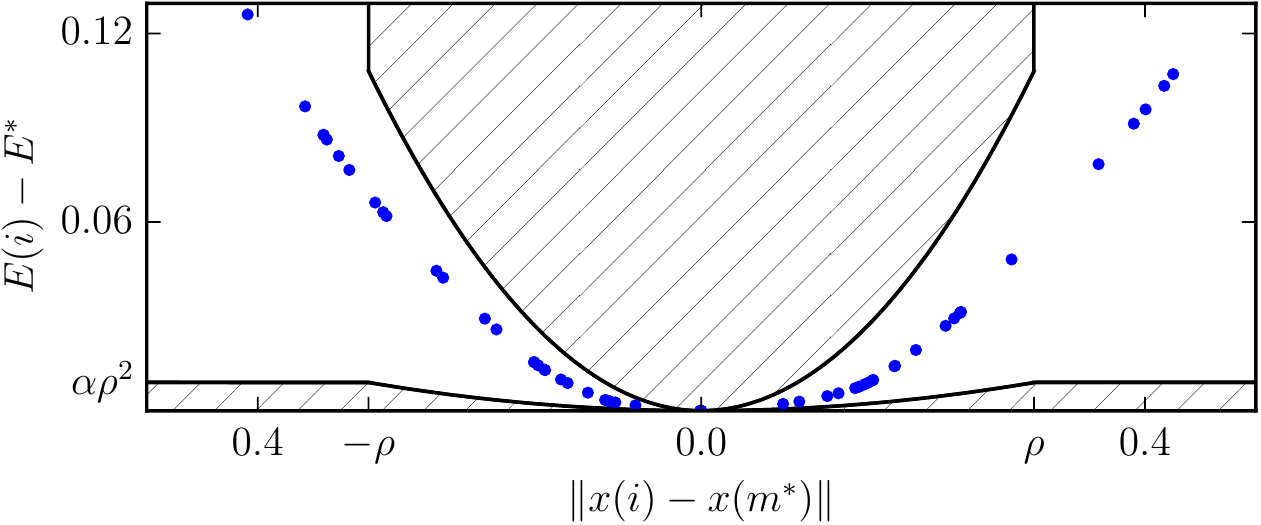}
\caption{Illustrating the parameters $\alpha$, $\beta$ and $\rho$ of Theorem~\ref{thm:main}. Here we draw $N = 101$ samples uniformly from $[-1,1]$ and compute their energies, plotted here as the series of points. Theorem~\ref{thm:main} states that their exists $\alpha$, $\beta$ and $\rho$ such that irrespective of $N$, all energies (points) will lie in the envelope (non-hatched region).}
\end{center}
\end{figure}

\begin{proof}
We first show that the expected number of computed elements in $\mathcal{B}_d(x(m^{*}), N^{-\frac{1}{2d}})$ is $O(V_d[1]\delta_1  N^{\frac{1}{2}})$. When $N$ is sufficiently large, $f_X(x) \le \delta_1$ within $\mathcal{B}_d(x(m^{*}), N^{-\frac{1}{2d}})$. The expected number of samples in $\mathcal{B}_d(x(m^{*}), N^{-\frac{1}{2d}})$ is thus upper bounded by $\delta_1$ multiplied by the volume of the ball. But the volume of a ball of radius $N^{-\frac{1}{2d}}$ in $\mathbb{R}^d$ is $V_d[1]N^{-\frac{1}{2}}$.

In Lemma~\ref{lemma:finalpacking} we use a packing argument to show that the number of computed elements in the annulus $\mathcal{A}_d(x(m^{*}), N^{-\frac{1}{2d}}, \infty)$ is $O\left( d\left(\frac{4}{\alpha}\right)^d N^{\frac{1}{2}}\right)$, but we there assume that the medoid index $m^{*}$ is the first element in $\texttt{shuffle}(\{1,\ldots, N \})$ on line 3 of \trimed{} and thus that the medoid energy is known from the first iteration $(E^{cl} = E^*)$. We now extend Lemma~\ref{lemma:finalpacking} to the case where the medoid is not the first element processed. We do this by showing that w.h.p. an element with energy very close to $E^*$ has been computed after $N^{-\frac{1}{2}}$ iterations of~\trimed{}, and thus that the bounds on numbers of computed elements obtained using the packing arguments underlying Lemma~\ref{lemma:finalpacking} are all correct to within some small factor after $N^{-\frac{1}{2}}$ iterations.

The probability of a sample lying within radius $N^{-\frac{2}{3d}}$ of $x(m^{*})$ is $\Omega(\delta_0 N^{-\frac{2}{3}})$, and so the probability that none of the first $N^{\frac{1}{2}}$ samples lies within radius $N^{-\frac{2}{3d}}$ is $O((1 - \delta_0 N^{-\frac{2}{3d}})^{N^{\frac{1}{2}}})$ which is $O(\frac{1}{N})$. Thus w.h.p. after $N^{\frac{1}{2}}$ iterations of~\trimed{}, $E^{cl}$ is within  $\beta N^{-\frac{4}{3d}}$ of $E^{*}$, which means that the radii of the balls used in the packing argument are overestimated by at most a factor $N^{-\frac{1}{3d}}$. Thus w.h.p. the upper bounds obtained with the packing argument are correct to within a factor $1 + N^{-\frac{1}{3}}$. The remaining $O(\frac{1}{N})$ cases do not affect the expectation, as we know that no more than $N$ elements can be computed. 
\end{proof}

\begin{lemma}[Packing beyond the vanishing radius]
\label{lemma:finalpacking}
If we assume~\eqref{strongcon} from Theorem~\ref{thm:main} and that the medoid index $m^{*}$ is the first element processed by \trimed{}, then the number of elements computed in $\mathcal{A}_d(x(m^{*}), N^{-\frac{1}{2d}}, \infty)$ is $O\left( d\left(\frac{4}{\alpha}\right)^d N^{\frac{1}{2}}\right)$.
\end{lemma}
\begin{proof}
Follows from Lemmas~\ref{lemma:midring} and~\ref{lemma:outerring}.
\end{proof}

\begin{lemma}[Packing from the vanishing radius $N^{-\frac{1}{d}}$ to $\rho$]
\label{lemma:midring}
If we assume~\eqref{strongcon} from Theorem~\ref{thm:main} and that the medoid index $m^{*}$ is the first element processed in \trimed{}, then the number of computed elements in $\mathcal{A}(x(m^{*}), N^{-\frac{1}{2d}},\rho)$ is $O(d\left(\frac{4}{\alpha }\right)^d N^{\frac{1}{2}})$.
\end{lemma}
\begin{proof}
According to Assumption~\ref{strongcon}, an element at radius $r < \rho$ has surplus energy at least $\alpha r^2$. This means that, assuming that the medoid has already been computed, an element computed at radius $r$ will be surrounded by an exclusion zone of radius $\alpha r^2$ in which no element will subsequently be computed. We will use this fact to upper bound the number of computed elements in $\mathcal{A}(x(m^{*}), N^{-\frac{1}{2d}},\rho)$, firstly by bounding the number in an annulus of inner radius $r$ and width $\alpha r^2$, that is the annulus $\mathcal{A}_d(x(m^{*}),r, r + \alpha r^2)$, and then summing over concentric rings of this form which cover $\mathcal{A}(x(m^{*}), N^{-\frac{1}{2d}},\rho)$. Recall that the number of computed elements in $\mathcal{A}_d(x(m^{*}),r, r + \alpha r^2)$ is denoted by $N_c(x(m^{*}), r, r + \alpha r^2)$. 

We use Lemma~\ref{lem:annuluspacking} to bound $N_c(x(m^{*}), r, r + \alpha r^2)$,
\begin{align*}
N_c(x(m^{*}), r, r + \alpha r^2) & \le (d + 1)^2\left(\frac{4}{\sqrt{3}}\right)^d \frac{\alpha  r^2 (r + \alpha  r^2)^{d-1}}{\left(\alpha  r^2\right)^d}\\
& \le (d + 1)^2\left(\frac{4}{\sqrt{3}}\right)^d \left( 1 + \frac{1}{\alpha r} \right)^{d-1} \\
& \le (d + 1)^2\left(\frac{4}{\sqrt{3}}\right)^d \left( \max{\left(2, \frac{2}{\alpha r}\right)} \right)^{d-1}\\
& \le (d + 1)^2\left(\frac{4}{\sqrt{3}}\right)^d \left( \max{\left(2^{d-1}, \left(\frac{2}{\alpha r}\right)^{d-1}\right)} \right)\\
& \le (d + 1)^2\left(\frac{4}{\sqrt{3}}\right)^d \left(2^{d-1} + \left(\frac{2}{\alpha r}\right)^{d-1} \right)\\
& \le (d + 1)^2\left(\frac{8}{\sqrt{3}}\right)^d  + (d + 1)^2\left(\frac{8}{\sqrt{3}}\right)^d  \left(\frac{1}{\alpha  r}\right)^{d-1}\\
\end{align*}
Let $r_0 = N^{-\frac{1}{2d}}$ and  $r_{i+1} = r_i + \alpha  r_i^2$, and let $T$ be the smallest index $i$ such that $r_i \le \rho$. With this notation in hand, we have
\begin{equation*}
N_c(x(m^{*}), N^{-\frac{1}{2d}},\rho) \le \sum_{i = 0}^{T} N_c (x(m^{*}), r_i, \alpha  r_i + r_i^2).
\end{equation*}
The summation on the right-hand side can be upper-bounded by an integral. Using that the difference between $r_i$ and $r_{i+1}$ is $\alpha r_i^2$, we need to divide terms in the sum by $\alpha r_i^2$ when converting to an integral. Doing this, we obtain,
\begin{align*}
N_c(x(m^{*}), N^{-\frac{1}{2d}},\rho) & \le \int_{N^{-\frac{1}{2d}}}^{\rho + \alpha \rho^2} N_c (x(m^{*}), r, \alpha  r^2) dr \\
& \le \mbox{const }  + (d + 1)^2\left(\frac{8}{\sqrt{3}}\right)^d \left(\frac{1}{\alpha }\right)^d \int_{N^{-\frac{1}{2d}}}^{\infty} r^{-(1+d)} dr \\
& \le \mbox{const }  + (d + 1)\left(\frac{4}{\alpha }\right)^d  N^{\frac{1}{2}}.
\end{align*}
This completes the proof, and provides the hidden constant of complexity as $(d+1)\left(\frac{4}{\alpha }\right)^d$. Thus larger values for $\alpha $ should result in fewer computed elements in the annulus $\mathcal{A}_d(x(m^{*}),r, r + \alpha  r^2)$, which makes sense given that large values of $\alpha $ imply larger surplus energies and thus larger elimination zones.
\end{proof}

\begin{lemma}[Packing beyond $\rho$]
\label{lemma:outerring}
If we assume~\eqref{strongcon} from Theorem~\ref{thm:main} and that the medoid index $m^{*}$ is the first element processed by \trimed{}, then the number of computed elements in $\mathcal{A}_d(x(m^{*}),\rho, \infty)$ is less than $( 1 + 4E^{*}/(\alpha \rho^2))^d$.
\end{lemma}
\begin{proof}
Recall that we at assuming $m^{*} = 1$, that is that the medoid is the first element processed in \trimed{}. All elements beyond radius $2 E^{*}$ are eliminated by type 1 eliminations (Figure~\ref{triangles2}), which provides the first inequality below. Then, as the excess energy is at least $\epsilon = \alpha  \rho^2 $ for all elements beyond radius $\rho$ of $x(m^{*})$, we apply Lemma~\ref{lemma:lazypack} with $\epsilon = \alpha  \rho^2/2$ to obtain the second inequality below,
\begin{align*}
N_c(m(x), \rho, \infty) & \le N_c(m(x), \rho, 2E^{*}) \\
& \le \frac{(2 E^{*} + \frac{1}{2}\alpha  \rho^2)^d}{(\frac{1}{2}\alpha  \rho^2)^d} \\
& \le \left( 1 + \frac{4E^{*}}{\alpha \rho^2}\right)^d.
\end{align*}
\end{proof}

\begin{lemma}[Annulus packing]
\label{lem:annuluspacking}
For  $0 \le r$ and $0 < \epsilon  \le w$. If
\begin{equation*}
\mathcal{X} \subset \mathcal{A}_d(0, r,r+ w),
\end{equation*}
where
\begin{equation}
\label{condition1}
\forall x \in \mathcal{X}, \mathcal{B}_d(x,\epsilon ) \cup \mathcal{X} = \{x\},
\end{equation}
then, 
\begin{equation*}
|\mathcal{X}| \le \left(d + 1\right)^2 \left( \frac{4}{\sqrt{3}} \right)^d \frac{w \left( r + w\right)^{d-1}}{\epsilon^d}.
\end{equation*}
\end{lemma}
\begin{proof}
The condition~\eqref{condition1} implies,
\begin{equation}
\label{reworded}
\forall x, x' \in \mathcal{X} \times \mathcal{X}, \mathcal{B}\left(x,\frac{\epsilon}{2} \right) \cup  \mathcal{B}\left(x',\frac{\epsilon}{2} \right) = \emptyset.
\end{equation}
Using that $\epsilon \in (0, w]$ and Lemma~\ref{corenlius}, one can show that for all $x \in \mathcal{A}(0, r, r+ w)$, 
\begin{equation}
\label{lowerperx}
\volume \left( \mathcal{B}\left(x,\frac{\epsilon}{2} \right)\cap \mathcal{A}(0, r, r+w) \right) > \frac{1}{d+1}\left(\frac{3}{4} \right)^{\frac{d}{2}}V_d\left[\frac{\epsilon}{2}\right]
\end{equation}
Combining~\eqref{reworded} with~\eqref{lowerperx} we have,
\begin{equation}
\label{lower101}
\volume \left( \displaystyle \bigcup_{x\in \mathcal{X}} \mathcal{B}\left(x,\frac{\epsilon}{2} \right) \cap \mathcal{A}(0, r, r+w)  \right) > \frac{V_d\left[1\right]}{d+1}\left(\frac{\sqrt{3}}{4}\right)^{d} |\mathcal{X}|  \epsilon^d.
\end{equation} 
Letting $S_d\left[\epsilon\right]$ denote the surface area of a $\mathcal{B}(0, \epsilon)$, it is easy to see that
\begin{equation}
\label{annvolume}
\volume \left( \mathcal{A}(0, r, r+w) \right) <   S_d\left[1\right]w \left(r + w\right)^{d-1}.
\end{equation}
Combining~\eqref{lower101} with~\eqref{annvolume} we get, 
\begin{equation*}
\frac{V_d\left[1\right]}{d+1}\left(\frac{\sqrt{3}}{4}\right)^d |\mathcal{X}|  \epsilon^d < S_d\left[1\right]w \left(r + w\right)^{d-1}.
\end{equation*}
which combined with the fact that 
\begin{align*}
\frac{S_d\left[1\right] }{ V_d\left[1\right]} &= \left(\frac{\frac{dV_d}{dr}}{V_d}\right)_{r=1} \\
&= d,
\end{align*}
provides us with,
\begin{equation*}
|\mathcal{X}| \le \left(d + 1\right)^2 \left( \frac{4}{\sqrt{3}} \right)^d \frac{w \left( r + w\right)^{d-1}}{\epsilon^d}.
\end{equation*}
\end{proof}

\begin{lemma}[Volume of ball intersection]
\label{corenlius}
For $x_0, x_1 \in \mathbb{R}^d$ with $\|x_0 - x_1\| = 1, $ 
\begin{equation*}
\frac{\volume \left( \mathcal{B}_d\left(x_0, 1 \right) \cap \mathcal{B}_d\left(x_1, 1 \right)  \right) }{ \volume \left( \mathcal{B}_d\left(x_0, 1 \right) \right) } \ge \frac{1}{d+1} \left(\frac{3}{4} \right)^\frac{d}{2}.
\end{equation*}
\end{lemma}
\begin{proof}
Let $V_d\left[r\right]$ denote the volume of $\mathcal{B}_d(0,r)$. It is easy to see that,
\begin{align*}
\volume \left( \mathcal{B}_d\left(x_0, 1 \right) \cap \mathcal{B}_d\left(x_1, 1 \right)  \right)  &= 2\int_0^{\frac{1}{2}} V_{d-1}\left[ \sqrt{x(2-x)}\right] dx \\
& \ge 2 \int_0^{\frac{1}{2}} V_{d-1} \left[ \sqrt{\frac{3}{2} x}\right] dx \\
& \ge 2 V_{d-1}\left[1\right]  \int_0^{\frac{1}{2}} \left(\frac{3}{2} x\right)^{\frac{d-1}{2}} dx \\
& \ge 2 V_{d-1}\left[1\right] \left(\frac{3}{2}\right)^{\frac{d-1}{2}} \left(\frac{2}{d+1}\right) \left(\frac{1}{2}\right)^{\frac{d+1}{2}} \\
& \ge V_{d-1}\left[1\right] \left(\frac{3}{2}\right)^{\frac{d-1}{2}} \left(\frac{2}{d+1}\right) \left(\frac{1}{2}\right)^{\frac{d-1}{2}} \\
& \ge V_{d-1}\left[1\right] \left(\frac{3}{4}\right)^{\frac{d-1}{2}} \left(\frac{2}{d+1}\right). \\
\end{align*}
Using that $\displaystyle \frac{V_{d-1}\left[1\right]}{V_{d}\left[1\right]}  > \frac{1}{\sqrt{\pi}}$  , we divide the intersection volume through by $V_{d}\left[1\right]$ to obtain, 
\begin{align*}
\frac{\volume \left( \mathcal{B}_d\left(x_0, 1 \right) \cap \mathcal{B}_d\left(x_1, 1 \right)  \right) }{ \volume \left( \mathcal{B}_d\left(x_0, 1 \right) \right) } & \ge \left(\frac{3}{4}\right)^{\frac{d-1}{2}} \left(\frac{2}{\sqrt{\pi}(d+1)}\right)\\
& \ge \frac{1}{d+1} \left(\frac{3}{4}\right)^{\frac{d}{2}}
\end{align*}
\end{proof}

\begin{lemma}[Packing balls in a ball]
\label{lemma:lazyfill}
The number of non-intersecting balls of radius $\epsilon$ which can be packed into a ball of radius $r$ in $\mathbb{R}^d$ is less than $\left(\frac{r}{\epsilon}\right)^d$
\end{lemma}
\begin{proof}
The technique used here is a loose version of that used in proving Lemma~\ref{lem:annuluspacking}. The volume of $\mathcal{B}_d(0, \epsilon)$ is a factor $\left(r/\epsilon\right)^d$ smaller than that of $\mathcal{B}_d(0, r)$. As the balls of radius $\epsilon$ are non-overlapping, the volume of their union is simply the sum of their volumes. The result follow from the fact that the union of the balls of radius $\epsilon$ is contained within the ball of radius $r$.
\end{proof}

\begin{lemma}[Packing points in a ball]
\label{lemma:lazypack}
Given $\mathcal{X} \subset \mathcal{B}_d(0, r)$ such that no two elements of $\mathcal{X}$ lie within a distance of $\epsilon$ of each other, $|\mathcal{X}| < \left(\frac{2r + \epsilon}{\epsilon}\right)^d$.
\end{lemma}
\begin{proof}
As no two elements lie within distance $\epsilon$ of each other, balls of radius $\epsilon/2$ centred at elements are non-intersecting. As each of the balls of radius $\epsilon/2$ centred at elements of $\mathcal{X}$ lies entirely within $\mathcal{B}_d(0,r + \epsilon/2)$, we can apply Lemma~\eqref{lemma:lazyfill}, arriving at the result. 
\end{proof}

\section{Pseudocode for \texttt{trikmeds}}
\label{app:trikmeds}
In Alg.~\eqref{alg::triangle_k_medoids} we present \texttt{trikmeds}. It is decomposed into algorithms for initialisation~\eqref{alg::initialise}, updating medoids~\eqref{alg::update_medoids}, assigning data to clusters~\eqref{alg::assign_to_clusters} and updating bounds on the \trimed{} derived bounds~\eqref{alg::update_sum_bounds}. Table~\ref{beebop} summarised all of the variables used in \texttt{trikmeds}.

When there are no distance bounds, the location of the bottleneck in terms of distance calculations depends on $N/K^2$. If $N/K \gg K$, the bottleneck lies in updating medoids, which can be improved through the strategy used in~\trimed{}. If $N/K \ll K$, the bottleneck lies in assigning elements to clusters, which is effectively handled through the approach of~\cite{elkan_2003_kmeansicml}.

\begin{table}[ht]
\caption{Table Of Notation For \texttt{trikmeds}}
\label{beebop}
\begin{center}
\begin{tabular}{r r l }
\toprule
$N$ & : & number of training samples \\
$i$ & : & index of a sample, $i \in \{1, \ldots, N\}$\\
$x(i)$ & : & sample $i$\\
$K$ & : & number of clusters \\
$k$ & : & index of a cluster, $k \in \{1, \ldots, K\}$\\
$m(k)$ & : & index of current medoid of cluster $k$, $m(k) \in \{1, \ldots, N\}$ \\
$c(k)$ & : & current medoid of cluster $k$, that is $c(k) = x(m(k))$\\
$n_1(i)$ & : & cluster index of centroid nearest to $x(i)$\\
$a(i)$ & : & cluster to which $x(i)$ is currently assigned\\
$d(i)$ & : & distance from $x(i)$ to $c(a(i))$\\
$v(k)$ & : & number of samples assigned to cluster $k$\\
$V(k)$ & : & number of samples assigned to a cluster of index less than $k+1$\\
$l_c(i,k)$ & : & lowerbound on distance from $x(i)$ to $m(k)$\\
$l_s(i)$ & : & lowerbound on $\sum_{i' : a(i') = a(i)} \|x(i') - x(i)\| $ \\
$p(k)$ & : & distance moved (teleported) by $m(k)$ in last update\\
$s(k)$ & : & sum of distances of samples in cluster $k$ to medoid $k$ \\
\bottomrule
\end{tabular}
\end{center}
\label{tab:TableOfNotationForMyResearch}
\end{table}

\begin{algorithm}
\begin{algorithmic}
\STATE \texttt{initialise}$()$
\WHILE{not converged}
\STATE \texttt{update-medoids}$()$
\STATE \texttt{assign-to-clusters}$()$
\STATE \texttt{update-sum-bounds}$()$
\ENDWHILE
\end{algorithmic}
\caption{\texttt{trikmeds}}
\label{alg::triangle_k_medoids}
\end{algorithm}

\begin{algorithm}
\begin{algorithmic}
\ACOMM{Initialise medoid indices, uniform random sample without replacement (or otherwise)}
\STATE $\{m(1), \ldots, m(K)\} \gets \texttt{uniform-no-replacement}( \{1, \ldots, N\} )$
\FOR{$k = 1:K$}
\ACOMM{Initialise medoid and set cluster count to zero}
\STATE $c(k) \gets x(m(k))$
\STATE $v(k) \gets 0$
\ACOMM{Set sum of in-cluster distances to medoid to zero}
\STATE $s(k) \gets 0$
\ENDFOR
\FOR{$i = 1: N$}
\FOR{$k = 1: K$}
\ACOMM{Tightly initialise lower bounds on data-to-medoid distances}
\STATE $l_c(i, k) \gets \|x(i) - c(k)\|$
\ENDFOR
\ACOMM{Set assignments and distances to nearest (assigned) medoid}
\STATE $a(i) \gets \argmin_{k \in \{1, \ldots, K\}}{l_c(i,k)}$
\STATE $d(i) \gets l_c(i, a(i))$
\ACOMM{Update cluster count}
\STATE $v(a(i)) \gets v(a(i)) + 1$
\ACOMM{Update sum of distances to medoid}
\STATE $s(a(i)) \gets s(a(i)) + d(i)$
\ACOMM{Initialise lower bound on sum of in-cluster distances to $x(i)$ to zero}
\STATE $l_s(i) \gets 0$
\ENDFOR
\STATE $V(0) \gets 0$
\FOR{$k = 1:K$}
\ACOMM{Set cumulative cluster count}
\STATE $V(k) \gets V(k-1) + v(k) $
\ACOMM{Initialise lower bound on in-cluster sum of distances to be tight for medoids}
\STATE $l_s(m(k)) \gets s(k)$
\ENDFOR
\ACOMM{Make clusters contiguous}
\STATE \texttt{contiguate}$()$
\end{algorithmic}
\caption{\texttt{initialise}}
\label{alg::initialise}
\end{algorithm}

\begin{algorithm}
\begin{algorithmic}
\FOR{$k = 1 : K$}
\FOR{$i = V(k-1):V(k) - 1$}
\ACOMM{If the bound test cannot exclude $i$ as $m(k)$}
\IF{$l_s(i) < s(k)$}
\ACOMM{Make $l_s(i)$ tight by computing and cumulating all in-cluster distances to $x(i)$,}
\STATE $l_s(i) \gets 0$
\FOR{$i' = V(k-1):V(k) - 1$}
\STATE $\tilde{d}(i') \gets \|x(i) - x(i')\| $
\STATE $l_s(i) \gets l_s(i) + \tilde{d}(i') $
\ENDFOR 
\ACOMM{Re-perform the test for $i$ as candidate for $m(k)$, now with exact sums. If $i$ is the new best candidate, update some cluster information}
\IF{$l_s(i) < s(k)$}
\STATE $s(k) \gets l_s(i)$
\STATE $m(k) \gets i$
\FOR{$i' = V(k-1):V(k) - 1$}
\STATE $d(i') \gets \|x(i) - x(i')\|$
\ENDFOR
\ENDIF
\ACOMM{Use computed distances to $i$ to improve lower bounds on sums for all samples in cluster $k$  (see Figure X)}
\FOR{$i' = V(k-1):V(k) - 1$}
\STATE $l_s(i') \gets \max{(l_s(i'), |\tilde{d}(i')v(k) - l_s(i) | )}$
\ENDFOR
\ENDIF
\ENDFOR
\ACOMM{If the medoid of cluster $k$ has changed, update cluster information}
\IF{$m(k) \not= V(k-1)$}
\STATE $p(k) \gets \|c(k) - x(m(k))\| $
\STATE $c(k) \gets x(m(k))$
\ENDIF
\ENDFOR
\end{algorithmic}
\caption{\texttt{update-medoids}}
\label{alg::update_medoids}
\end{algorithm}

\begin{algorithm}
\begin{algorithmic}
\ACOMM{Reset variables monitoring cluster fluxes, } 
\FOR{$k = 1 : K$}
\ACOMM{the number of arrivals to cluster $k$,}
\STATE $\Delta_{n-in}(k) \gets 0$
\ACOMM{the number of departures from cluster $k$,}
\STATE $\Delta_{n-out}(k) \gets 0$
\ACOMM{the sum of distances to medoid $k$ of samples which leave cluster $k$}
\STATE $\Delta_{s-out}(k) \gets 0$
\ACOMM{the sum of distances to medoid $k$ of samples which arrive in cluster $k$}
\STATE $\Delta_{s-in}(k) \gets 0$
\ENDFOR
\FOR{$i = 1 : N$}
\ACOMM{Update lower bounds on distances to medoids based on distances moved by medoids}
\FOR{$k = 1 : K$}
\STATE $l(i,k) = l(i,k) - p(k)$
\ENDFOR
\ACOMM{Use the exact distance of current assignment to keep bound tight (might save future calcs)}
\STATE $l(i,a(i)) = d(i)$
\ACOMM{Record current assignment and distance}
\STATE $a_{old} = a(i)$
\STATE $d_{old} = d(i)$
\ACOMM{Determine nearest medoid, using bounds to eliminate distance calculations}
\FOR{$k = 1 : K$}
\IF{$l(i,k) < d(i)$}
\STATE $l(i,k) \gets \|x(i) - c(k)\|$
\IF{$l(i,k) < d(i)$}
\STATE $a(i) = k$
\STATE $d(i) = l(i,k)$
\ENDIF
\ENDIF
\ENDFOR
\ACOMM{If the assignment has changed, update statistics}
\IF{$a_{old} \not= a(i)$}
\STATE $v(a_{old}) = v(a_{old}) - 1$
\STATE $v(a(i)) = v(a(i)) + 1$
\STATE $l_s(i) = 0$
\STATE $\Delta_{n-in}(a(i)) = \Delta_{n-in}(a(i)) + 1$
\STATE $\Delta_{n-out}(a_{old}) = \Delta_{n-out}(a_{old}) + 1$
\STATE $\Delta_{s-in}(a(i)) = \Delta_{s-in}(a(i)) + d(i)$
\STATE $\Delta_{s-out}(a_{old}) = \Delta_{s-out}(a_{old}) + d_{old}$
\ENDIF
\ENDFOR
\ACOMM{Update cumulative cluster counts}
\FOR{$k = 1:K$}
\STATE $V(k) \gets V(k-1) + v(k) $
\ENDFOR
\STATE \texttt{contiguate}$()$
\end{algorithmic}
\caption{\texttt{assign-to-clusters}}
\label{alg::assign_to_clusters}
\end{algorithm}

\begin{algorithm}
\begin{algorithmic}
\FOR{$k = 1:K$}
\ACOMM{Obtain absolute and net fluxes of energy and count, for cluster $k$}
\STATE $\mathcal{J}^{abs}_s(k) = \Delta_{s-in}(k) + \Delta_{s-out}(k)$
\STATE $\mathcal{J}^{net}_s(k) = \Delta_{s-in}(k) - \Delta_{s-out}(k)$
\STATE $\mathcal{J}^{abs}_n(k) = \Delta_{n-in}(k) + \Delta_{n-out}(k)$
\STATE $\mathcal{J}^{net}_n(k) = \Delta_{n-in}(k) - \Delta_{n-out}(k)$
\FOR{$i = V(k-1):V(k)-1$}
\ACOMM{Update the lower bound on the sum of distances} 
\STATE $l_s(i) \gets l_s(i) - \min (\mathcal{J}^{abs}_s(k) - \mathcal{J}^{net}_n(k)d(i), \mathcal{J}^{abs}_n(k)d(i) - \mathcal{J}^{net}_s(k)) $
\ENDFOR
\ENDFOR
\end{algorithmic}
\caption{\texttt{update-sum-bounds}}
\label{alg::update_sum_bounds}
\end{algorithm}

\begin{algorithm}
\begin{algorithmic}
\STATE // This function performs an in place rearrangement over of variables $a,d,l,x$ and $m$
\STATE // The permutation applied to $a,d,l$ and $x$ has as result a sorting by cluster,
\STATE // $a(i) = k$ if $i \in \{V(k-1), V(k) \}$ for $ k \in \{1, \ldots, K \}$
\STATE // and moreover that the first element of each cluster is the medoid,
\STATE // $m(k) = V(k-1)$ for $k \in \{1, \ldots, K\}$
\end{algorithmic}
\caption{\texttt{contiguate}}
\label{alg::contiguate}
\end{algorithm}

\section{Datasets}
\label{app:datasets}
\begin{itemize}
\item \emph{Birch1}, \emph{Birch2} : Synthetic 2-D datasets available from \url{https://cs.joensuu.fi/sipu/datasets/}
\item \emph{Europe} : Border map of Europe available from \url{https://cs.joensuu.fi/sipu/datasets/}
\item \emph{U-Sensor Net} : Undirected 2-D graph data. Points drawn uniformly from unit square, with an undirected edge connecting points when the distance between them is less than $1.25\sqrt{N}$
\item \emph{D-Sensor Net} : Directed 2-D graph data. Points drawn uniformly from unit square, with directed edge connecting points when the distance between them is less than $1.45\sqrt{N}$, direction chosen at random.
\item \emph{Europe rail} : The European rail network, the shapefile is available at \url{http://www.mapcruzin.com/free-europe-arcgis-maps-shapefiles.htm}. We extracted edges from the shapefile using \texttt{networkx} available at \url{https://networkx.github.io/}.
\item \emph{Pennsylvania road} The road network of Pennsylvania, the edge list is available directly from \url{https://snap.stanford.edu/data/}
\item \emph{Gnutella} Peer-to-peer network data, available from \url{https://snap.stanford.edu/data/}
\item \emph{MNIST (0)}  The `0's in the MNIST training dataset. 
\item \emph{Conflong} The conflongdemo data is available from \url{https://cs.joensuu.fi/sipu/datasets/}
\item \emph{Colormo} The colormoments data is available at \url{http://archive.ics.uci.edu/ml/datasets/Corel+Image+Features}
\item \emph{MNIST50} The MNIST dataset, projected into 50-dimensions using a random projection matrix where each of the $784 \times 50 $ elements in the matrix is i.i.d. $\mathcal{N}(0,1)$.
\item \emph{S1, S2, S3, S4, A1, A2, A3} All of these synthetic datasets are available from \url{https://cs.joensuu.fi/sipu/datasets/}.
\item \emph{thyroid, yeast, wine, breast, spiral} All of these real world datasets are available from \url{https://cs.joensuu.fi/sipu/datasets/}.
\end{itemize}

\section{Scaling with dimension of \texttt{TOPRANK} and \texttt{TOPRANK2}}
\label{app:topscale}
Recall the assumption~\eqref{toprankass} made for the \texttt{TOPRANK} and \texttt{TOPRANK2} algorithms. The assumption states that as one approaches the minimum energy $E^{*}$ from above, the density of elements decreases. In other words, the lowest energy elements stand out from the rest and are not bunched up with very similar energies. 

Consider the case where elements are points in $\mathbb{R}^d$. Suppose that the density $f_X$ of points around the medoid is bounded by $0 < \rho_0 \le f_X \le \rho_1$, and that the energy grows quadratically in radius about the medoid.  Then, as the number of points at radius $\epsilon$ is $O(\epsilon^{d-1})$, the density (by energy) of points at radius $\epsilon$ is $O(\epsilon^{d-2})$. Thus for $d = 1$ the assumption for \texttt{TOPRANK} and \texttt{TOPRANK} does not hold, which results in poor performance for $d=1$. For $d = 2$, the assumption holds, as the density (by energy) of points is constant. For $d \ge 2$, as $d$ increases the energy distribution becomes more and more favourable for  \texttt{TOPRANK} and \texttt{TOPRANK2}, as the low ranking elements become more and more distinct with low energies becoming less probable. This explains the observation that \texttt{TOPRANK} scales well with dimension in Figure~\ref{rainbow}.

\section{Example where geometric median is a poor approximation of medoid}
There is no guarantee that the geometric median is close to the set medoid. Moreover, the element in $\mathcal{S}$ which is nearest to $g(\mathcal{S})$ is not necessarily the medoid, as illustrated in the following example. Suppose $S = \{x(1), \ldots, x(20)\} \subset \mathbb{R}^2$, with $x(i) = (0,1)$ for $i \in \{1, \ldots, 9\}$, $x(i) = (0,-1)$ for $i \in \{10, \ldots, 18\}$, $x(19) = (1/2,0)$ and $x(20) = (-1/2, 0)$. The geometric median is $(0,0)$ and the nearest points to the geometric median, $x(19)$ and $x(20)$ have energy $1 + 18\sqrt{3}/2 \approx 16.6 $.  However, points $\{x(1), \ldots, x(18)\}$ have energy $2\sqrt{3}/2 + 9 = 10.7$. Thus by choosing a point in $\mathcal{S}$ which is nearest to the geometric median, one is choosing the element with the highest energy, the opposite of the medoid. 

Note the above example appears to violate the assumptions required for $O(N^{3/2})$ convergence of \trimed{}, as it requires that the probability density function vanishes at the distribution median. Indeed, in $\mathbb{R}^d$ it is the case that if the $O(N^{3/2})$ assumptions are satisfied, the set medoid converges to the geometric median, and so the geomteric median is a good approximation.  We stress however that the geometric median is only relevant in vector spaces.

\section{Miscellaneous}
Figure~\ref{last_resort} illustrates the idea behind algorithm \trimed{}, comments in the caption.
\begin{figure}[t!]
\begin{center}
\includegraphics[width=0.8\textwidth]{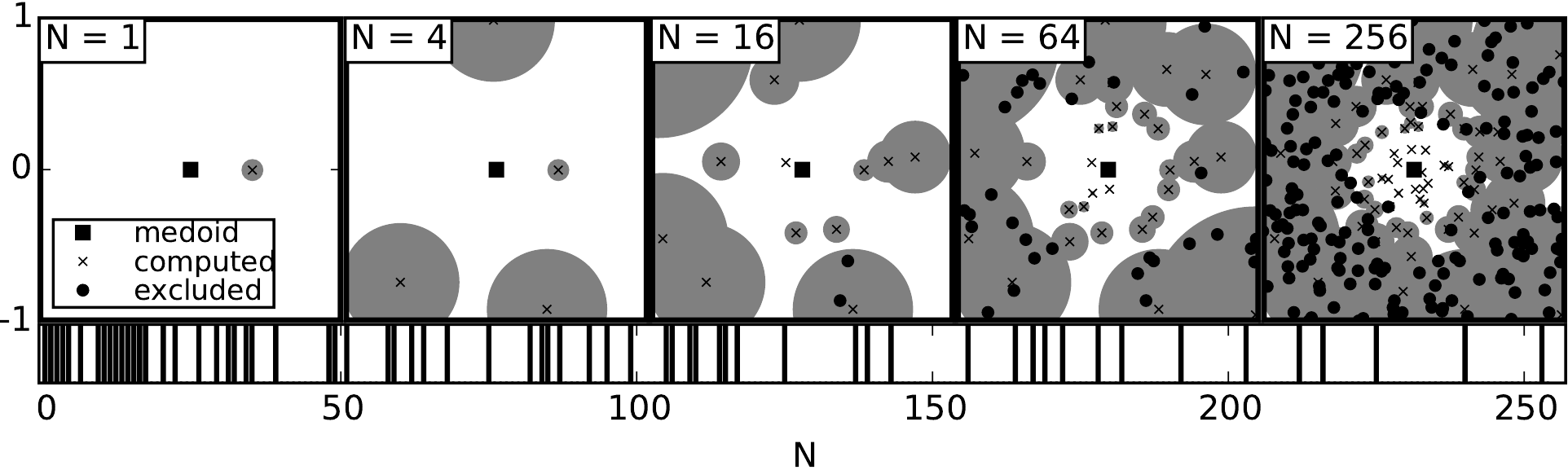}
\caption{Eliminating samples as potential medoids using only type 1 elimination, where we assume that the medoid and its energy $E^{*}$ are known, and so the radius of the exclusion ball of an element $x$ is $E(x) - E^{*}$. Uniformly sampling from $[-1,1] \times [-1,1]$, energies are computed only if the sample drawn does not lie in the exclusion zone (union of balls). If the energy at $x$ is computed, the exclusion zone is augmented by adding $\mathcal{B}_d(x, E(x) - E^{*})$. Top left to right: the distribution of samples which are computed and excluded. Bottom: the times at which samples are computed. We prove that probability of computation at time $n$ is $O(n^{-\frac{1}{2}})$. }
\label{last_resort}
\end{center}
\end{figure}

\end{document}